\documentclass{article}

\usepackage[utf8]{inputenc} %
\usepackage[T1]{fontenc}    %
\usepackage{amsthm}

\newtheorem{lemma}{Lemma}[section]
\usepackage{url}            %
\usepackage{booktabs}       %
\usepackage{amsfonts}
\usepackage{bbm}%
\usepackage{nicefrac}       %
\usepackage{microtype}      %
\usepackage{xcolor}         %
\usepackage{fleqn, tabularx}
\usepackage{multirow}
\usepackage[export]{adjustbox}
\usepackage{amsmath}
\usepackage{amssymb}
\usepackage{colortbl}
\usepackage{wrapfig}
\usepackage[noend]{algorithmic}
\usepackage{xcolor}
\usepackage{quoting}
\newtheorem{assumption}{Assumption}
\newtheorem{theorem}{Theorem}
\newtheorem{corollary}{Corollary}
\usepackage{amsmath, amsthm, bbm, mathtools, commath}
\usepackage{datetime}
\usepackage{verbatim}
\usepackage[font=small]{subfig}
\usepackage{colortbl}
\usepackage[ruled]{algorithm2e}
\usepackage{arydshln} %
\newcommand{\facc}[2]{{#1}{\scriptsize$\pm${#2}}}
\newcommand{\bfacc}[2]{\textbf{{#1}{\scriptsize$\pm${#2}}}}

\newtheorem{definition}{Definition}

\newcommand{\printfnsymbol}[1]{%
  \textsuperscript{\@fnsymbol{#1}}%
}
\makeatother

\definecolor{Gray}{gray}{0.9}



\usepackage{titlesec}
\titlespacing\section{0pt}{0pt plus 2pt minus 2pt}{0pt plus 2pt minus 2pt}
\titlespacing\subsection{0pt}{3pt plus 4pt minus 2pt}{0pt plus 2pt minus 2pt}
\titlespacing\subsubsection{0pt}{3pplus 4pt minus 2pt}{0pt plus 2pt minus 2pt}

\usepackage[colorlinks=true,linkcolor=black,citecolor=black]{hyperref}

\usepackage[skins,theorems]{tcolorbox}
\usepackage[numbers]{natbib}

\usepackage[final]{neurips_2022}

\title{Decoupled Self-supervised Learning \\for  Non-Homophilou Graphs}
\usepackage{lineno}
\author{%
  Teng Xiao$^1$, Zhengyu Chen$^2$, Zhimeng Guo$^1$, Zeyang Zhuang$^{3}$, Suhang Wang$^1$ \\
  $^1$The Pennsylvania State University, $^2$Zhejiang University, $^3$Tongji University\\
  \texttt{\{tengxiao,zhimeng,szw494\}@psu.edu,chenzhengyu@zju.edu.cn}\\
  \texttt{zeyangzhuang0315@gmail.com}
}
\begin{document}
\maketitle

\begin{abstract}
This paper studies the problem of conducting self-supervised learning for node representation learning on  graphs. Most existing self-supervised learning methods assume the graph is homophilous, where linked nodes often belong to the same class or have similar features. However, such assumptions of homophily do not always hold in real-world graphs. We address this problem by developing a decoupled self-supervised learning (DSSL) framework for graph neural networks. DSSL imitates a generative process of nodes and links from latent variable modeling of the semantic structure, which decouples different underlying semantics between different neighborhoods into the self-supervised learning process. Our DSSL framework is agnostic to the encoders and does not need prefabricated augmentations,
thus is flexible to different graphs. To effectively optimize the framework,  we derive the evidence lower bound of the self-supervised objective and develop a scalable training algorithm with variational inference. We provide a theoretical analysis to justify that DSSL enjoys the better downstream performance. Extensive experiments on various types of graph benchmarks demonstrate that our proposed framework can  achieve better performance compared with competitive  baselines.
\end{abstract}

\section{Introduction}
Graph-structured data is ubiquitous in the real world, such as social networks, knowledge graphs, and molecular structures. In recent years, graph neural networks (GNNs)~\cite{hamilton2017inductive,DBLP:conf/iclr/KipfW17,DBLP:conf/iclr/VelickovicCCRLB18,chen2022ba,xu2018representation} have been proven to be powerful in node representation learning over graph-structured data. Typically, GNNs are trained with annotated labeled data in a supervised manner. However, collecting labeled data is expensive and impractical in many applications, especially for those requiring domain knowledge, such as medicine and chemistry~\cite{zitnik2018prioritizing,hu2020strategies}. Moreover, supervised learning may suffer from problems of less-transferrable, over-fitting, and poor generalization when the task labels are scarce~\cite{ericsson2021well,you2020graph}.

Recently, self-supervised learning (SSL) provides a promising learning paradigm that reduces the dependence on manual labels in the image domain~\cite{chen2020simple,gidaris2018unsupervised,grill2020bootstrap,chen2021pareto}. Compared to image data, there are unique challenges in designing self-supervised learning schemes for graph-structured data since nodes in the graph are correlated with each other rather than completely independent, and geometric structures are essential and heavily impact the performance in downstream tasks~\cite{DBLP:conf/iclr/KipfW17}. A number of recent works~\cite{velivckovic2018deep,hassani2020contrastive,you2020graph,zhu2021graph,thakoor2021large,zhang2021canonical,suresh2021adversarial} have studied graph self-supervised learning and confirm that it can learn transferrable and generalizable node representations without any labels. Typically, there are two main self-supervised schemes to capture structure information in graphs~\cite{velivckovic2018deep,you2020graph,suresh2021adversarial}. The first scheme is reconstructing the vertex adjacency following traditional network-embedding methods~\cite{kipf2016variational,grover2016node2vec,hamilton2017representation,hamilton2017inductive,tang2015line}, which learns an encoder that imposes the topological closeness of nodes in the graph structure on latent representations. The key assumption behind this scheme is that neighboring nodes have similar representations~\cite{velivckovic2018deep,tang2015line}. However, this assumption over-emphasizes proximity~\cite{velivckovic2018deep,you2020graph,tang2021graph} and does not always hold true for heterophilic and non-homophilous (mixing) graphs. In comparison, contrastive learning methods~\cite{velivckovic2018deep,hassani2020contrastive,you2020graph,zhu2021graph,thakoor2021large,zhang2021canonical} construct two graph views via the stochastic augmentation and then learns representations by contrasting views with information maximization principle. While these contrastive methods can capture structure information without directly emphasizing proximity, their performance relies on topology augmentation~\cite{zhu2021empirical,trivedi2021augmentations}. Importantly, conducting augmentation for non-homophilous  graphs is relatively difficult since linked nodes may be dissimilar, and nodes with high similarities might be farther away from each other. 
Hence, the above problems pose an important and challenging research question: \textit{How to design an effective self-supervised scheme for node representation learning in non-homophilous graphs?}

We approach this question by  taking advantage of neighborhood strategies of nodes for the self-supervised learning on non-homophilous graphs. Our key motivation is that nodes with similar neighborhood patterns should have similar representations. In other words, we expect that the neighborhood distributions can be exploited to distinguish node representations. Our assumption is more general than the standard homophily assumption as shown in Figure~\ref{fig:motivation}. For instance, while the gender prediction in common dating networks lacks homophily~\cite{altenburger2018monophily}, neighborhood distribution is very informative to the node gender labels, i.e., nodes with similar neighborhoods are likely to be similar.  


\begin{figure}[t]
\centering
    \includegraphics[width=0.850\textwidth]{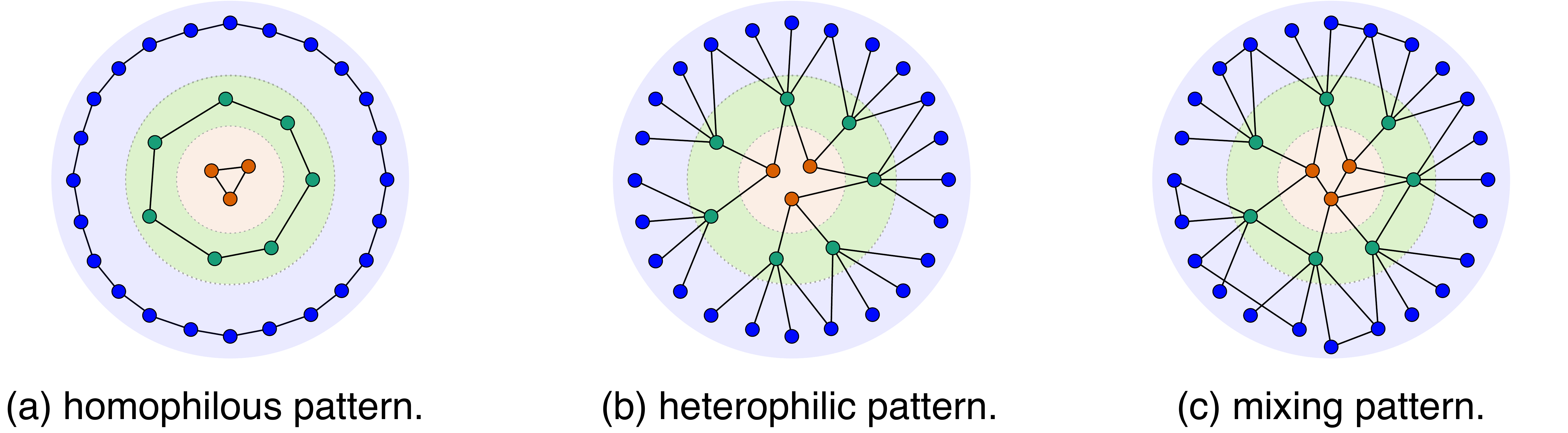}
    \vspace{-0.6em}
     \caption{An illustration examples of different types of graphs. Nodes with similar labels typically have  similar neighborhood patterns in all types of graph. This assumption is general than the standard homophily assumption where nodes with similar semantic label typically be linked with each other.}\label{fig:motivation}
\vspace{-1em}
\end{figure}

While the motivation is straightforward, we are faced with two main challenges. The first challenge is capturing the local neighborhood distribution in a self-supervised learning manner. The neighborhood distributions typically follow heterogeneous and diverse patterns. For instance, a node usually connects with others due to the complex interaction of different latent factors and therefore possesses distribution consisting of local mixing patterns wherein certain parts of the neighborhood are homophilous while others are non-homophilous. The lack of supervision obstructs us from modeling the distribution of neighborhoods. The second challenge is capturing the long-range semantic dependencies. As shown in Figure~\ref{fig:motivation}, in non-homophilous graphs, nodes with high semantic and local structural similarities might be farther away from each other. For this reason, global semantic information is the objective that we would incorporate for self-supervised learning.

This paper presents a new self-supervised framework, decoupled self-supervised learning (DSSL), to achieve a good balance between these two challenges. At the core of DSSL is the latent variables, which empower the model with the flexibility to decouple the heterogeneous and diverse patterns in local neighborhood distributions and capture the global semantic dependencies in a coherent and synergistic framework. Our contributions can be summarized as follows: (1) We propose a  DSSL for performing self-supervised learning on non-homophilous graphs, which can leverage both useful local structure and global semantic information.
(2) We develop an efficient training algorithm based on variational inference to simultaneously infer the latent variables and learn encoder parameters.
(3) We analyze the properties of DSSL and theoretically show that the learned representations can achieve better downstream performance. 
(4) We conduct experiments on real-world non-homophilous graphs, and the results demonstrate the effectiveness of our self-supervised learning framework.


\section{Related Work}
\textbf{Non-homophilous Graphs.} 
Non-homophilous  is known in many  settings such as online transaction networks~\cite{pandit2007netprobe},  dating networks~\cite{altenburger2018monophily} and molecular networks~\cite{zhu2020beyond}. Recently, various GNNs~\cite{pei2019geom,zhu2020beyond,xiao2021learning,lim2021large,zhu2021graph,chien2020adaptive,suresh2021breaking,yang2021diverse} have been proposed to deal with non-homophilous graphs with different methods such as potential neighbor discovery~\cite{pei2019geom,jin2021node,jin2021universal}, adaptive message propagation~\cite{chien2020adaptive,xiao2021learning}, exploring high-frequency signals~\cite{bo2021beyond} and higher-order message passing~\cite{zhu2020beyond,chen2020simple}. Despite their success,  they typically consider the semi-supervised setting and are trained with task-specific labeled data, while in practice, labels are often limited, expensive, and even inaccessible. In contrast, in this paper,  we study the problem of self-supervised learning: learning  node representations without relying on labels.

\textbf{Self-supervised Learning on Graphs.} 
Self-supervised learning holds great promise for improving representations when labeled data are scarce~\cite{chen2020simple,gidaris2018unsupervised,grill2020bootstrap,he2020momentum}. Earlier combinations of GNNs and self-supervised learning involve GraphSAGE~\cite{hamilton2017inductive}, VGAE~\cite{kipf2016variational}, and Graphite~\cite{grover2019graphite}, which typically follow the traditional network-embedding methods~\cite{perozzi2014deepwalk,tang2015line,grover2016node2vec} and adopt the link reconstruction or random walk principle. Since these  methods over-emphasize node proximity  at the cost of structural information~\cite{velivckovic2018deep,suresh2021adversarial,you2020graph}, various graph contrastive learning methods have been proposed~\cite{velivckovic2018deep,hassani2020contrastive,you2020graph,zhu2021graph,thakoor2021large,mo2022simple}, which aim to learn representations by  contrasting representation under differently augmented views and have achieved promising performance. However, they heavily rely on complex data- or task-specific augmentations, which are prohibitively difficult for non-homophilous graphs. 
Our work differs from the above methods and aims to answer the question of how to design an effective self-supervised learning scheme for non-homophilous graphs.

\textbf{Disentangled graph learning.}  
Our work is also related to but different from existing disentangled graph learning that aims to decouple latent factors in the graph. vGraph~\cite{sun2019vgraph} considers a mixture process to define the conditional probability on each edge. vGraph is a classical shallow network embedding algorithm. However, we instead tackle the problem of self-supervised learning with GNNs. There are a couple of works that explore the disentangled factors in the node-level~\cite{ma2019disentangled,liu2020independence}, edge-level~\cite{zhao2022exploring} and graph-level~\cite{yang2020factorizable}. Whereas these methods require task-specific labels that can be extremely scarce. By contrast, we address the problem of learning node representations on non-homophilous graphs without labels. Disentangled contrastive learning~\cite{li2021disentangled}  learns disentangled graph representations without labeled graphs. \cite{xu2021self} proposes the self-supervised graph-level representation learning with disentangled local and global structure~\cite{xu2021self}. Nevertheless, their goal is to conduct graph-level classification tasks, but we tackle a node-level representation problem. \cite{zhao2021graph,ding2022structural,wang2022clusterscl} consider adding a clustering layer or a prototype clustering component to capture the global information. However,  they are still based on contrastive learning~\cite{zhao2021graph,ding2022structural} or supervised learning~\cite{wang2022clusterscl}. By contrast, our framework is an unsupervised generative model and does not rely on graph augmentations and downstream labels. Moreover, we focus on non-homophilous graphs where connected nodes may not be similar to each other by decoupling the local diverse neighborhood context.
\section{Decoupled Self-supervised Learning}
In this section, we describe our problem setting with respect to self-supervised learning and demonstrate our approach. An illustration of generative and inference processes is depicted in Figure~\ref{fig:framework} (a).
\begin{figure}[t]
\centering
    \includegraphics[width=0.93\textwidth]{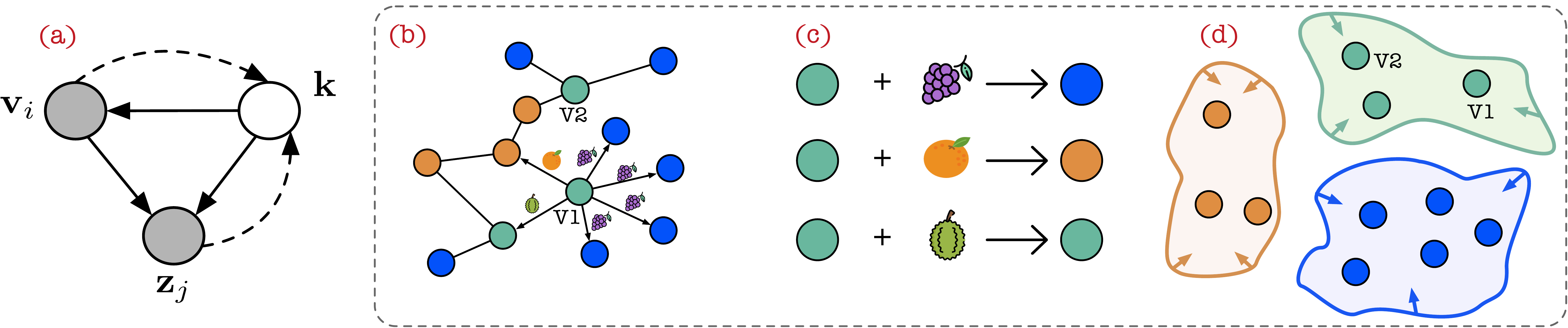}
    \vskip -0.5em
     \caption{An illustration of  (a) a graphical model for DSSL. The  discrete latent variable $k$ is used to instantiate a latent mixture-of-Gaussians
$\mathbf{v}_{i}$, which is then decoded to $\mathbf{z}_{j}$ and (b) a toy subgraph example of non-homophilous graphs where color denotes the class of the node. (c) Our model encodes the semantic shift by decoupling structure since different each node has different latent factors to make connections to its different neighbor, and (d) captures global semantic structure via the semantic clusters, which seeks to push the representations of $v_{1}$ and $v_{2}$ to their closest prototypes.}\label{fig:framework}
\vspace{-2em}
\end{figure}
\label{sec:method}
\subsection{Problem Formulation}
We consider  a graph $G=(\mathcal{V}, \mathcal{E})$, where $\mathcal{V}$ is a set of $|\mathcal{V}|=N$ nodes and $\mathcal{E} \subseteq \mathcal{V} \times \mathcal{V}$ is a set of $|\mathcal{E}|$ edges between nodes. $\mathbf{A} \in$ $\{0,1\}^{N \times N}$ is the adjacency matrix of $G$. The $(i, j)$-th element $\mathbf{A}_{i j}=1$  if there exists an edge ($v_i, v_j$) between node $v_{i}$ and $v_{j}$, otherwise $\mathbf{A}_{i j}=0$.  The  matrix $\mathbf{X} \in \mathbb{R}^{N \times D}$ describes node features. The $i$-th row of $\mathbf{X}$, i.e., $\mathbf{x}_{i}$, is the feature  of node $v_{i}$. Given the graph  $\mathcal{G}=(\mathbf{X}, \mathbf{A})$, the objective of self-supervised node representation learning is to learn an encoder function $f_{{\theta}}(\mathbf{X}, \mathbf{A}): \mathbb{R}^{N \times N} \times \mathbb{R}^{N \times D}  \rightarrow \mathbb{R}^{N \times D^{\prime}} $  where ${\theta}$ denotes its parameters, such that the  representation of node $v_{i}$: $f_{{\theta}}(\mathbf{X}, \mathbf{A})[i]$, can be used for downstream tasks such as node classification.

\subsection{The Probabilistic Framework}
In this section, we introduce our framework, decoupled self-supervised learning, which can learn meaningful node representations on non-homophilous graphs by capturing the intrinsic graph structure. The core idea of DSSL is to model the distributions of node neighbors via a mixture of generative processes in the representation space. Specifically, we model the generation of neighbors by assuming each node has latent heterogeneous factors which are utilized to make connections to its different neighbors. Intuitively,  the factor denotes various reasons behind why two nodes are connected in non-homophilous graphs. For instance, two nodes in a school network will be connected depending on some factors such as colleagues, friends, or classmates; in protein networks, even if they do not have similar features, different amino acid types are likely to be connected due to various interactions. 

Formally, let $f_{{\theta}}(\mathbf{X}, \mathbf{A})[i]=\mathbf{v}_{i}$ and  $f_{{\theta}}(\mathbf{X}, \mathbf{A})[j]=\mathbf{z}_{j}$ be the representations  of nodes $v_{i}$ and $v_{j}$, respectively. Here we utilize different notations, i.e., $\mathbf{v}$ and $\mathbf{z}$ to distinguish between central and neighbor nodes since
each node plays two roles: the node itself and a specific neighbor of other nodes. Our goal is to find the encoder parameter $\theta$ which maximizes the likelihood of distribution $p(\mathbf{z}_{j}|\mathbf{v}_{i};\theta)$  on central node  and its neighbor. To model the unobserved factors, we associate every node $\mathbf{v}_{i}$ with a discrete latent variable $k$ to indicate to which factor $\mathbf{v}_{i}$ has.  Assume that there are  $K$ factors in total, the log-likelihood of node neighbors  $\mathbf{v}_{i}$ can be written by marginalizing out the latent variables: 
\begin{linenomath}
\small
\begin{align}
\mathcal{L}_{{DSSL}}({\theta})= \frac{1}{|{N(i)}|}\sum_{j\in N(i)}\log [p_{\theta}(\mathbf{z}_{j}| \mathbf{v}_{i})]= \frac{1}{|{N(i)}|} \sum_{j\in N(i)} \log \Big[\sum_{k=1}^{K} p_{\theta}(\mathbf{z}_{j}|\mathbf{v}_{i},k)p_{\theta}(k|\mathbf{v}_{i})\Big], \label{generative}
\end{align}
\end{linenomath}
where $N(i)$ is the set of out-neighbors of $v_{i}$, the  distribution 
$p_{\theta}(k|\mathbf{v}_{i})$  indicates the 
assignment of latent semantics  over central node representation $\mathbf{v}_{i}$, and $p_{\theta}(\mathbf{z}_{j}|\mathbf{v}_{i},k)$ is the probability that node ${v}_{i}$ and its neighbor ${v}_{j}$ are connected under factor $k$. Unlike previous works~\cite{tang2015line,grover2016node2vec,kipf2016variational}, which directly encourage nearby nodes to have similar representations, we provide an alternative way to model node neighbors and seek to decouple their latent relationship without any prior on neighbor partitions. The similar high-level ideas of modeling latent variables between two nodes has also been explored in local augmentation in graphs (LA-GCN)~\cite{liu2022local} and vGraph~\cite{sun2019vgraph}. Different from them, we consider modeling discrete relations in the learned representation space for self-supervised learning on non-homophilous graphs. In Eq.~\eqref{generative}, probabilities $p_{\theta}(k|\mathbf{v}_{i})$ and $p_{\theta}(\mathbf{z}_{j}|\mathbf{v}_{i},k)$  are not specified, and involve discrete latent variables. To make it solvable, we introduce the following generative process.

Let $\boldsymbol{\mu}_{k}$ and $\boldsymbol{\Sigma}_{k}$ be the mean and variance of the latent mixture component $k$, and $\pi_{k}$ be its corresponding mixture probability. The generation of a node  and its neighbor shown in Figure~\ref{fig:framework} (a) typically involves  three steps: 
(1) draw a latent  variable $k$ from a categorical distribution $p(k)$ on all mixture components, where $p(k)$ is usually defined as uniform distribution $p(k)=\frac{1}{K}$ for unknown graphs and better generalization, (2) draw the central node representation $\mathbf{v}_{i}$ from the Gaussian distribution  $p_{\theta}(\mathbf{v}_{i}|k)=\mathcal{N}(\mathbf{v}_{i};\boldsymbol{\mu}_{k}$, $\boldsymbol{\Sigma}_{k}$), and (3) draw the neighbor representation $\mathbf{z}_{j}$ from  Gaussian distribution 
$p_{\theta}(\mathbf{z}_{j}|\mathbf{v}_{i},k)=\mathcal{N}(\mathbf{z}_{j}; \boldsymbol{\mu}_{z_j},\boldsymbol{\Sigma}_{j})$ where  the mean depends both on central representation and latent variable: $\boldsymbol{\mu}_{z_j}=\mathbf{v}_{i}+\beta g_{\theta}({k})$. Here the projector $g_{\theta}(\cdot)$ denotes another network that embeds latent variable $k$ to the representation space, and $\beta$ is the parameter that controls the strength of interpolation. In practice, to reduce complexity, we consider using isotropic Gaussian, i.e., $\forall k:~ \Sigma_{k}=\mathbf{I}\sigma_{1}^{2}$ and $\forall j:~ \Sigma_{j}=\mathbf{I}\sigma_{2}^{2}$ where $\mathbf{I}$ is the
identity matrix, $\sigma_{1}$ and $\sigma_{2}$ are hyperparameters. In alignment with  this generative process, the joint distribution of the edge and latent variable can be 
written as:
\begin{linenomath}
\begin{align}
p_{\theta}(\mathbf{v}_{i},\mathbf{z}_{j},k)=p_{\theta}(\mathbf{z}_{j}| \mathbf{v}_{i},k)p_{\theta}(\mathbf{v}_{i}|k)p(k). \label{eq:genk}
\end{align}
\end{linenomath}
Intuitively, $p_{\theta}(\mathbf{v}_{i}|k)$ can be viewed as the probability for node representation under the $k$-th mixture component. Regarding   $p_{\theta}(\mathbf{z}_{j}|\mathbf{v}_{i},k)=\mathcal{N}(\mathbf{z}_{j}; \boldsymbol{\mu}_{z_j},\boldsymbol{\Sigma}_{j})$, this formulation is inspired by knowledge graph embedding~\cite{bordes2013translating}, where two entities should be closed to each other under a certain relation operation. However, unlike knowledge graph embedding, the relational context is a latent variable and not observed in our setting. Intuitively, instead of enforcing the exact representation alignment of two linked nodes~\cite{kipf2016variational,grover2016node2vec,hamilton2017representation,hamilton2017inductive,tang2015line}, our design can relax this strong homophily assumption and account for the semantic shift of representations between the central node and its neighbors.

Now, the posterior probability $p_{\theta}(k|\mathbf{v}_{i})$ in  Eq.~\eqref{generative} can be derived by using Bayesian rules as follows:  
\begin{linenomath}
\small
\begin{align}
p_{\theta}(k|\mathbf{v}_{i})=\frac{p_{\theta}(\mathbf{v}_{i}|k)p(k)}{\sum_{k^{\prime}=1}^{K}p_{\theta}(\mathbf{v}_{i}|k')p(k')}=\frac{ \mathcal{N}(\mathbf{v}_{i}; \boldsymbol{\mu}_{k},\mathbf{I}\sigma_{1}^{2})p(k)}{\sum_{k^{\prime}=1}^{K} \mathcal{N}(\mathbf{v}_{i}; \boldsymbol{\mu}_{k'},\mathbf{I}\sigma_{1}^{2})p(k')}, \label{Eq:posterior}
\end{align}
\end{linenomath}%
where $\boldsymbol{\mu}=\{\boldsymbol{\mu}_{k}\}_{k=1}^{K}$ can be treated as a set of trainable prototype representations, which 
are the additional distribution parameters. This posterior probability represents the soft semantic assignment of the learned representation to the prototypes,  which makes the node with similar semantic properties be close to its prototype and  encode both the inter- and intra-cluster variation. Substituting  $p_{\theta}(k|\mathbf{v}_{i})$ and $p_{\theta}(\mathbf{z}_{j}|\mathbf{v}_{i},k)$ in   Eq.~\eqref{generative} with the specified  probabilities, we get final objective $\mathcal{L}_{{DSSL}}({\theta},\boldsymbol{\mu})$.

\subsection{Evidence Lower Bound}
Given the  framework above,  we are interested in: (i) learning the model parameters $\theta$ and $\boldsymbol{\mu}$ by maximizing the log-likelihood $\mathcal{L}_{{DSSL}}({\theta},\boldsymbol{\mu})$, and (ii) inferring the posterior of latent variable $k$ for each observed links. However, it is computationally intractable to directly solve these two problems due to the latent variables. To solve this, we resort to amortized variational inference methods~\cite{DBLP:journals/corr/KingmaW13}, and maximize the evidence lower-bound (ELBO) of Eq.~\eqref{generative}, i.e.,
$\mathcal{L}_{\mathrm{DSSL}}({\theta},\boldsymbol{\mu})\geq \mathcal{L}_{\mathrm{DSSL}}({\theta},\phi,\boldsymbol{\mu})$:
\begin{linenomath}
\small
\begin{align}
\mathcal{L}_{\mathrm{DSSL}}(\theta,\phi,\boldsymbol{\mu}) =\frac{1}{|{N(i)}|} \sum_{j\in N(i)} \mathbb{E}_{q_{\phi}(k| \mathbf{v}_{i}, \mathbf{z}_{j})}[\log p_{\theta}(\mathbf{z}_{j}| \mathbf{v}_{i}, k)+\log p_{\theta}(k| \mathbf{v}_{i})]+\mathcal{H}(q_{\phi}(k| \mathbf{v}_{i},\mathbf{z}_{j})), 
\label{ELBO}
\end{align}
\end{linenomath}%
where $q_{\phi}({k}|\mathbf{v}_{i},\mathbf{z}_{j})$ is the introduced  variational distribution parameterized by $\phi$ and $\mathcal{H}(\cdot)$ is the entropy operator. $\mathbf{z}_{j}$ is included in this variational posterior, so the inference is also conditioned on the neighborhood information. We derive this ELBO in Appendix~\ref{app:ELBO}. Maximizing this ELBO w.r.t. $\{\theta,\phi,\boldsymbol{\mu}\}$ is equivalent to (i) maximizing $\mathcal{L}_{\mathrm{DSSL}}({\theta},\boldsymbol{\mu})$ and to (ii) make variational  $q_{\phi}({k}|\mathbf{v}_{i},\mathbf{z}_{j})$ be close to true posterior. Plugging the parameterized probabilities into this ELBO, we obtain the following loss to minimize. See Appendix~\ref{app:loss} for corresponding derivations.
\begin{linenomath}
\small
\begin{align}
\mathcal{L}=\frac{1}{|{N(i)}|}& \sum_{j\in N(i)} \mathbb{E}_{q_{\phi}(k| \mathbf{v}_{i}, \mathbf{z}_{j})}\big[ \|\mathbf{v}_{i}+\beta g_{\theta}(k)-\mathbf{z}_{j}\|_{2}^{2}\big]-{\sigma}_{2}^{2} \mathbb{E}_{q_{\phi}(k|\mathbf{v}_{i})}\Big[\log \frac{\exp (\mathbf{v}_{i}^{\top} \cdot \boldsymbol{\mu}_{k} / {\sigma}_{1}^{2})}{\sum_{k^{\prime}=1}^{K}  \exp (\mathbf{v}_{i}^{\top} \cdot \boldsymbol{\mu}_{k’} / {\sigma}_{1}^{2})}\Big] \nonumber \\
&-\mathcal{H}(q_{\phi}(k| \mathbf{v}_{i},\mathbf{z}_{j})),
\label{loss}
\end{align}
\end{linenomath}%
where $q_{\phi}(k|\mathbf{v}_{i})={1}/{|{N(i)}|} \sum\nolimits_{j\in N(i)} q_{\phi}(k| \mathbf{v}_{i}, \mathbf{z}_{j})$ is the posterior probability of semantic assignment  for central node $v_{i}$, by aggregating all its neighbors. Thus, the first term (denoted as $\mathcal{L}_{local}$) in the loss encourages the model to reconstruct the local neighbors while considering different semantic shifts captured by latent variable $k$ (see Figure~\ref{fig:framework} (b)). The second term (denoted as $\mathcal{L}_{global}$) encourages the model to perform clustering with learned representation where possible, i.e., seeking to push the representation $\mathbf{v}_{i}$ to its closest prototype cluster (see Figure~\ref{fig:framework} (c)). The final entropy term makes the model choose to have high entropy over $q_{\phi}(k| \mathbf{v}_{i}, \mathbf{z}_{j})$ such that all of the  $K$-channel losses must be low. Overall, this loss derived from ELBO can capture global semantic similarities over neighborhoods and learn to decouple different latent patterns in the local neighbors.


Regarding the variaitonal distribution $q_{\phi}({k}|\mathbf{v}_{i},\mathbf{z}_{j})$, we model it as categorical distribution since  $k$ is a discrete multinomial variable. Specifically, the representations $\mathbf{v}_{i}$ and $\mathbf{z}_{j}$ are encoded to a combined representation and then   $q_{\phi}({k}|\mathbf{v}_{i},\mathbf{z}_{j})$ is determined by an output softmax  inference head as follows:
\begin{linenomath}
\small
\begin{align}
q_{\phi}({k}|\mathbf{v}_{i},\mathbf{z}_{j})=\frac{\exp (h_{\phi}([\mathbf{v}_{i};\mathbf{z}_{j}])[k])}{\sum_{k'=1}^{K}\exp (h_{\phi}([\mathbf{v}_{i};\mathbf{z}_{j}])[k'])},
\end{align}
\end{linenomath}%
where $h_{\phi}$ denotes the  inference network parameterized by $\phi$ and $[\cdot,\cdot]$ can be the element-wise product or concatenation operation. $h_{\phi}([\mathbf{v}_{i};\mathbf{z}_{j}])[k]$ indicates the $k^{th}$ element, i.e., the logit corresponding the latent context $k$. Instead of introducing variational parameters individually, we consider the amortization inference~\cite{DBLP:journals/corr/KingmaW13,xiao2021general}, which fits a shared network to calculate each local parameter.

For the expectation terms in Eq.~\eqref{loss}, back-propagation through the discrete variable $k$ is not directly feasible. We alleviate this by adopting the  Straight-Through Gumbel-Softmax estimator~\cite{jang2016categorical}, which provides a continuous differentiable approximation for drawing samples from a categorical distribution. Specifically, for each sample, a latent cluster vector $\mathbf{c}\in (0,1)^{K}$ is drawn from:
\begin{linenomath}
\small
\begin{align}
\mathbf{c}[k]=\frac{\exp( ( h_{\phi}([\mathbf{v}_{i};\mathbf{z}_{j}])[k] +\boldsymbol{\epsilon}[k]) / \gamma)}{\sum_{k'=1}^{K} \exp (( h_{\phi}([\mathbf{v}_{i};\mathbf{z}_{j}])[k] +\boldsymbol{\epsilon}[k']) / \gamma)},
\end{align}
\end{linenomath}%
where $\boldsymbol{\epsilon}[k]$ is i.i.d drawn from the $\operatorname{Gumbel}(0,1)$ distribution and $\gamma$ is a temperature. With this reparameterization trick,  we can obtain the surrogate $\mathbb{E}_{q_{\phi}({k}|\mathbf{v}_{i},\mathbf{z}_{j})}[ \|\mathbf{v}_{i}+\beta g_{\theta}(k)-\mathbf{z}_{j}\|_{2}^{2}] \simeq \mathbb{E}_{\boldsymbol{\epsilon}}[ \|\mathbf{v}_{i}+\beta g_{\theta}(\mathbf{c})-\mathbf{z}_{j}\|_{2}^{2}]$  and
the gradients are estimated with Monte Carlo. The expectation term over $q_{\phi}(k|\mathbf{v}_{i})$ can be similarly estimated. Then,  $\{\theta,\phi,\boldsymbol{\mu}\}$ in  Eq.~\eqref{loss}  can be efficiently  solved by gradient descent.

\subsection{Algorithm Optimization}
The overall optimization involves simultaneously training (1) the encoder $f_{\theta}$, (2) the  projector $g_{\theta},(3)$ the inference predictor $h_{\phi}$ and (4) the prototypes $\boldsymbol{\mu}$. The most canonical way to update the parameters is stochastic gradient descent. However, we observe that stochastically updating all parameters suffers from two problems: (1) The objective admits trivial solutions, e.g., outputting the same representation for all nodes in the optimization process. (2) Updating prototypes without any constraints will lead to a degenerate solution, i.e., all nodes are assigned to a single  cluster. 

To address the issue of trivial solutions, inspired by the recent works~\cite{he2020momentum,grill2020bootstrap}, we consider an asymmetric encoder architecture that includes online and target encoders. Specifically, for each  node pair $(v_{i},v_{j})$, the online encoder $f_{\theta}$ produces the representation of the central node $\mathbf{v}_{i}$; while the target encoder $f_{\xi}$ is used to produce the representation of its neighbor $\mathbf{z}_{j}$. Importantly, the gradient of loss is only used to update the online encoder  $f_{\theta}$ while being blocked in the target encoder. The weights of the target encoder $\xi$ are  updated via the exponential moving average (EMA) of the online encoder $\theta$:
\begin{linenomath}
\small
\begin{align}
    [ \theta, \phi, \boldsymbol{\mu}] \leftarrow  [ \theta, \phi, \boldsymbol{\mu}]-\boldsymbol{\Gamma}\left(\nabla_{ \theta, \phi, \boldsymbol{\mu}} \mathcal{L}\right),~~~~~ \xi \leftarrow \tau \xi+(1-\tau) \theta,
\end{align}
\end{linenomath}%
where $\boldsymbol{\Gamma}(\cdot)$ indicates a stochastic optimizer and $\tau \in[0,1]$ is the target decay rate. This update introduces an asymmetry between two encoders that prevents collapse to trivial solutions~\cite{grill2020bootstrap,thakoor2021large}.
To alleviate the second issue, besides the stochastic update, we also  apply a global update for the  prototype vectors $\boldsymbol{\mu}=\left\{\boldsymbol{\mu}_{i}\right\}_{i=1}^{K}$ at the end of each training epoch to avoid a degenerate solution: 
\begin{linenomath}
\small
\begin{align}
    \boldsymbol{\mu}_{k}=\frac{\sum_{i=1}^{N}\pi_{i} (k)\cdot \mathbf{v}_{i}}{\|\sum_{i=1}^{N}\pi_{i} (k)\cdot \mathbf{v}_{i} \|_{2}^{2}},~~ \text{where}~~ \pi_{i} (k)={1}/{|{N(i)}|} \sum\nolimits_{j\in N(i)} q_{\phi}(k| \mathbf{v}_{i}, \mathbf{z}_{j}). \label{Eq:global}
\end{align}
\end{linenomath}%
The derivation is provided in Appendix~\ref{app:global}.  Intuitively, $\pi_{i} (k)$ reflects the  degree of relevance of node $v_{i}$ to the $k^{th}$ prototype. Instead of only updating the prototypes in a mini-batch, we also aggregate all the representations as the prototype based on the soft assignment probability at the end of the epoch. After the training is finished, we only keep the online encoder $f_{\theta}$ for the downstream task. Our full algorithm and network are provided in Appendix~\ref{app:alg}. The details about time complexity of DSSL is given in Appendix~\ref{sec:Time}, which scales linearly in the size of edges.

\subsection{Theoretical Analysis}
In this section, we provide an analysis of the proposed framework. We first present the connection between the proposed objective and the mutual information maximization, then show that the learned representations by our objective provably enjoy a good downstream performance. Due to the space limitation, all proofs of theorems and corollaries are provided in Appendix~\ref{app:theory}.

We denote the random variable of the input graph as $\mathcal{G}$ and the downstream label as $\mathbf{y}$. For clarity, we omit subscript $i$ in what follows. From an information-theoretic learning perspective, a desirable way is to  maximize the mutual information $I(\mathbf{v},\mathbf{y})$ between the representation $\mathbf{v}$ and downstream label $\mathbf{y}$. However,  due to the lack of the downstream label, self-supervised learning resorts to maximizing $I(\mathbf{v},\mathbf{s})$ where $\mathbf{s}$ is different designed self-supervised signal~\cite{tsai2020demystifying,bui2021exploiting,wang2021residual,zhang2021canonical,federici2019learning}. In our method, we have two self-supervised signals: the global semantic cluster information inferred by $k$ and the local structural roles captured by the representations of the neighbors $\mathbf{z}=\{\mathbf{z}_{i}| v_i \in {N(v)}\}$ of node $v$.

Then, we can  interpret our objective in Eq.~\eqref{loss} from the information maximization perspective:
\begin{theorem}\label{the:mutual}
 Optimizing local and global terms in Eq.~\eqref{loss} is equivalent to maximizing the mutual information between the representation $\mathbf{v}$ and global signal ${k}$ and maximizing the conditional mutual information between $\mathbf{v}$ and the local signal  $\mathbf{z}$, conditioned on global signal $k$. Formally, we have:
    \begin{align}
       \max_{\theta, \phi, \boldsymbol{\mu}} \mathcal{L}\Rightarrow  \max_{\mathbf{v}} {I}(\mathbf{v};k)+ {I}(\mathbf{v};\mathbf{z}|k)= {I}(\mathbf{v};k,\mathbf{z}).
    \end{align}
\end{theorem}
\vspace{-1em}
This theorem suggests that we essentially combine both local structure and global semantic information as the self-supervised signal and maximize the mutual information between the representation $\mathbf{v}$ and their joint distribution $(k,\mathbf{z})$. Next, we discuss how the learned representation affects the downstream task $\mathbf{y}$ based on the information bottleneck principle ~\cite{tsai2020demystifying,federici2019learning,wang2021residual,bui2021exploiting}. The rationality of self-supervised learning is that the task-relevant information lies mostly in the shared
information between the input and the self-supervised signals~\cite{tsai2020demystifying,federici2019learning,bui2021exploiting}. Specifically,  we formulate our lightweight and reasonable assumption below, which serves as a foundation for our analysis.

\begin{assumption}\label{assumption}
Nodes with similar labels should have similar "local structural roles" and "global semantic clusters." In this work, we equate "local structure" with the 1-hop neighborhood and "global semantic" with the clustering membership of a node. Formally, we have the task-relevant information $\mathbf{y}$ left in $\mathcal{G}$ except that in joint self-supervised signal $(\mathbf{z},{k})$ is relatively small: $I(\mathcal{G};\mathbf{y}|\mathbf{z},{k})\leq \epsilon$.
\end{assumption}
\vspace{-0.5em}
Intuitively, this assumption indicates that most of the task-relevant information in the graph is contained in the self-supervised signal (local neighborhood and global  semantic patterns). Based on this assumption, we have the following theorem, which reveals why the downstream tasks can benefit from the learned representations learned by our objective function.
\begin{theorem}\label{the:bound}
Let $\mathbf{v}_{\mathrm{joint}}=\arg \max_{\mathbf{v}} I(\mathbf{v};\mathbf{z},k),\mathbf{v}_{\mathrm{local}}=\arg \max_{\mathbf{v}} I(\mathbf{v};k)$, and $\mathbf{v}_{\mathrm{global}}=\arg \max_{\mathbf{v}}I(\mathbf{v};\mathbf{z})$. Formally, we have the following inequalities about the task-relevant information:
   \begin{align}
        I(\mathcal{G};\mathbf{y})=\max_{\mathbf{v}} I(\mathbf{v};\mathbf{y})\geq I(\mathbf{v}_{\mathrm{joint}};\mathbf{y})\geq \max (I(\mathbf{v}_{\mathrm{local}};\mathbf{y}),I(\mathbf{v}_{\mathrm{global}};\mathbf{y}))\geq I(\mathcal{G};\mathbf{y})-\epsilon. \label{Eq:proposition}
   \end{align}
\end{theorem}
\vspace{-1em}
Theorem~\ref{the:bound} shows that the gap of task-relevant information between supervised representation $\mathbf{v}_{sup}=\arg \max_{\mathbf{v}} I(\mathbf{v},\mathbf{y})$ and self-supervised representation $\mathbf{v}_{\mathrm{joint}}$ is $\epsilon$. Thus, we can guarantee a good downstream performance as long as the Assumption~\ref{assumption} is satisfied. It is noteworthy that jointly utilizing local structure and global semantics as the self-supervised signal is expected to contain more task information. As further enlightenment, we can relate  Eq.~\eqref{Eq:proposition} with the Bayes error rate~\cite{tsai2020demystifying}:
\begin{corollary}\label{the:bayes}
Suppose that downstream label $\mathbf{y}$ is a M-categorical random variable. Then we have the
upper bound for the downstream Bayes errors ${P}^{e}_{\mathbf{v}}=\mathbb{E}_{\mathbf{v}}\left[1-\max _{y \in \mathbf{y}} P(\hat{\mathbf{y}}=y| \mathbf{v})\right]$ on learned representation $\mathbf{v}$, where $\hat{\mathbf{y}}$ is the estimation for label from our downstream classifier:
\begin{align}
    \operatorname{Th}({P}_{\mathbf{v}_{joint}}^{e}) \leq \log 2+P_{\mathbf{v }_{\mathrm{sup}}}^{e} \cdot \log M+I(\mathcal{G} ; \mathbf{y}| \mathbf{z},k) \triangleq {\mathrm{RHS}}_{\mathbf{v}_{joint}},
\end{align}
where $\operatorname{Th}(x)=\min \{\max \{x, 0\}, 1-1 /|M|\}$ is a thresholded operation~\cite{tsai2020demystifying}. Similarly, we can obtain the  error upper bound of other representations $\mathbf{v}_{\mathrm{local}}$ and $\mathbf{v}_{\mathrm{global}}:$ ${\mathrm{RHS}}_{\mathbf{v}_{local}}$and ${\mathrm{RHS}}_{\mathbf{v}_{global}}$. Then, we have   inequalities on error upper bounds $: {\mathrm{RHS}}_{\mathbf{v}_{joint}}\leq \min ({\mathrm{RHS}}_{\mathbf{v}_{local}},{\mathrm{RHS}}_{\mathbf{v}_{global}})$.
\end{corollary}
\vspace{-0.5em}
This corollary says that our self-supervised signal has a tighter upper bound on the downstream Bayes error. Thus, we can expect that the representation learned by our objective function, which utilizes both local structure and global semantic information, has superior performance on downstream tasks.

\section{Experiments}
\label{sec:exp}
In this section, we empirically evaluate the proposed self-supervised learning method on several real-world graph datasets and analyze its behavior on graphs to gain further insights.

\subsection{Experimental Setup}
\textbf{Datasets.} We perform experiments on widely-used homophilic graph datasets: Cora, Citeseer, and Pubmed~\cite{sen2008collective}, as well as non-homophilic datasets: Texas, Cornell, Wisconsin~\cite{pei2019geom}, Penn94 and Twitch. Penn94 and Twitch are two relatively large non-homophilous graph datasets proposed by~\cite{rozemberczki2021multi,lim2021large}. We provide the detailed descriptions, statistics, and homophily measures of datasets in Appendix~\ref{app:dataset}.

\textbf{Baselines.} To evaluate the effectiveness of DSSL, we consider the following representative  unsupervised and self-supervised  learning methods for the node representation task, including Deepwalk~\cite{perozzi2014deepwalk}, LINE~\cite{tang2015line}, Struc2vec~\cite{ribeiro2017struc2vec}, GAE~\cite{kipf2016variational}, VGAE~\cite{kipf2016variational},  DGI~\cite{velivckovic2018deep}, GraphCL~\cite{you2020graph}, MVGRL~\cite{hassani2020contrastive} and BGRL~\cite{thakoor2021large}. The detailed description of baselines and implementations are given in Appendix~\ref{app:setup}.

\textbf{Evaluation Protocol.} 
We consider two types of downstream tasks: node classification and node clustering. For node classification, we follow the standard linear-evaluation protocol on graphs~\cite{velivckovic2018deep,thakoor2021large}, where a linear classifier is trained on top of the frozen representation, and test accuracy (ACC) is used as a proxy for representation quality. For all datasets, we adopt the similar random split with a train/validation/test split ratio of $60\%/20\%/20\%$ for the training of downstream linear classifier following~\cite{pei2019geom,lim2021large}. For the node clustering task, we perform $K$-means clustering on the obtained representations and set the number of clusters to the number of classes. We utilize the normalized mutual information (NMI)~\cite{vinh2010information} as the evaluation metric for clustering. 

\textbf{Setup.} For all self-supervised methods on graph neural networks, we consider a two-layer GCN~\cite{DBLP:conf/iclr/KipfW17} as the encoder unless otherwise stated, and  randomly initialize  parameters. We run experiments with 10 random splits and report the average performance. We select the best configuration of hyper-parameters based on accuracy on the validation. The detailed settings are given in Appendix~\ref{app:setup}.

\begin{table}[t]
\vspace{-0.05in}
\centering
\caption{
Experimental results (\%) with standard deviations on the node classification task. The best and second best
performance under each dataset are marked with boldface and underline, respectively.}\label{table:classification}
\vspace{0.02in}
\scalebox{0.83}{
\begin{tabular}{ccccccccc}
\toprule
 \textbf{Method}       & \textbf{Cora}      & \textbf{Citeseer}    & \textbf{Pubmed}   & \textbf{Texas}    & \textbf{Cornell}      & \textbf{Squirrel}     & \textbf{Penn94} & \textbf{Twitch} \\
\midrule
Deepwalk       &       \facc{77.14}{0.82}   &    \facc{67.85}{0.79} &     \facc{79.38}{1.22} &     \facc{42.31}{2.21}  &     \facc{41.55}{3.12}  &     \facc{37.54}{2.19}  &    \facc{56.13}{0.46} & \facc{66.37}{0.11}  \\
   LINE        & \facc{78.93}{0.55} & \facc{68.79}{0.41} & \facc{80.56}{0.92} & \facc{48.69}{1.39} & \facc{43.68}{2.17} & \underline{\facc{38.92}{1.58}} & \facc{57.59}{0.17} & \facc{67.23}{0.27} \\ 
   Struc2vec        & \facc{30.26}{1.52} & \facc{53.38}{0.62} & \facc{40.83}{1.85} & \facc{49.31}{3.22} & \facc{30.22}{5.87} & \facc{36.49}{1.15} & \facc{50.29}{0.31} & \facc{63.52}{0.35} \\
\midrule
  GAE        &  \facc{78.33}{0.27} & \facc{66.39}{0.24} & \facc{78.28}{0.77} &  {\facc{53.98}{3.22}}  & \facc{44.18}{3.56} & \facc{30.53}{1.33}  & \facc{58.11}{0.18} & \facc{67.98}{0.27}  \\ 
  VGAE        & \facc{80.59}{0.35} & \facc{69.90}{0.57} & \facc{81.33}{0.69} & \facc{50.27}{2.21} & \facc{48.73}{4.19} & \facc{29.13}{1.16} & \facc{58.29}{0.21} & \facc{65.09}{0.08}\\ 
  DGI        & \facc{84.17}{1.35} & \facc{71.80}{1.33} & \facc{81.65}{0.71} &  \underline{\facc{58.53}{2.98}} & \facc{45.33}{6.11} & \facc{26.44}{1.12} & \facc{53.68}{0.19} & \facc{66.97}{0.25}  \\ 
  \midrule
 GraphCL        & \underline{\facc{84.28}{0.91}} & \underline{\facc{72.46}{1.79}}  & \facc{81.96}{0.73} & \facc{48.67}{4.37} &  \facc{47.22}{4.50} & \facc{22.53}{0.98} & \facc{58.43}{0.31} & \underline{\facc{68.37}{0.16}} \\ 
  MVGRL        & \bfacc{85.21}{1.18} & \facc{72.13}{1.04}  & \underline{\facc{82.33}{0.88}}  & \facc{51.26}{0.38} & \underline{\facc{51.16}{1.67}} &  \facc{38.43}{0.87}  & \facc{57.22}{0.17} & \facc{66.03}{0.26}  \\ 
 BGRL        & \facc{83.29}{0.72}  & \facc{71.56}{0.87}  & \facc{81.34}{0.50}  & \facc{52.77}{1.98} & {\facc{50.33}{2.29}} & \facc{36.22}{1.97} & \underline{\facc{58.98}{0.13}} &  \facc{67.43}{0.22} \\ 
\midrule
DSSL  & \facc{83.06}{0.53} &  \bfacc{73.51}{0.64} &  \bfacc{82.98}{0.49} &  \bfacc{62.11}{1.53} &   \bfacc{53.15}{1.28}  &  \bfacc{40.51}{0.38} &  \bfacc{60.38}{0.32}  &    \bfacc{69.81}{0.17} \\ 
\bottomrule
\end{tabular}}
\vspace{-0.1in}
\end{table}

\subsection{Overall Performance Comparison}
In this section, we conduct experiments on real-world graphs compared to state-of-the-art methods. Table~\ref{table:classification} reports the average classification accuracy with the standard deviation on node classification after ten runs. Since node clustering results have a similar tendency, we provide them in Appendix~\ref{app:node-clutering}. From Table~\ref{table:classification}, we have the following observations: (1) Generally, our DSSL achieves the best performance on both
node clustering and classification tasks over the best baseline, demonstrating the effectiveness of DSSL on both homophilous and non-homophilous graphs and the robustness to different downstream tasks and graphs. Especially, DSSL achieves a relative improvement of over 6\% and 4\% on Texas and Squirrel compared to the best baselines  (2) Our DSSL cannot achieve the best performance on Cora. This is reasonable since Cora is highly homophilous (see Appendix~\ref{app:dataset}), and the core design of augmentation in contrastive learning methods such as MVGRL and GraphCL enables them to be effective on homophilous graphs but failed on low-homophily settings. This supports our motivation in the introduction that it is relatively difficult to design effective graph augmentation on non-homophilous graphs due to their heterogeneous and diverse patterns. (3) When compared with GAE, VGAE, and DGI,  DSSL consistently and significantly outperforms them. We deem that this improvement is mainly from the joint local semantic shift and global semantic clustering in DSSL, which is not included in existing methods.

\begin{figure*}[t]
	\centering
	\includegraphics[width=1.01\linewidth]{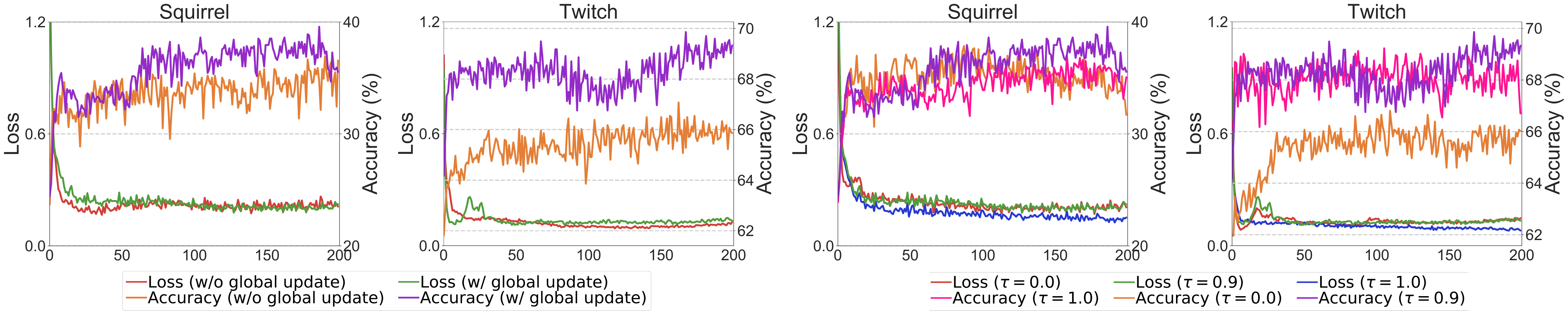}
	\caption{(Left) Performance w/ and w/o global update.
	(Right) Performance with varying $\tau$.}
	\label{fig:curves}
	\vskip -1.5em
\end{figure*}

\begin{wraptable}[15]{RT}{.4\linewidth}
		\small
		\centering
		\vskip -1.5em
		\caption{\small Node classification accuracy ($\%$) on the Texas dataset.}\label{tab_abl}
		\vskip -0.5em
		\resizebox{\linewidth}{!}{
			\begin{tabular}{l c }
			  \toprule[1.0pt]
			   		\textbf{Ablation}& \textbf{Accuracy}\\
				\toprule
				\textbf{A1}  w/o local loss  $\mathcal{L}_{local}$ & \facc{25.18}{1.31}   \\
			    \textbf{A2} w/o global loss $\mathcal{L}_{global}$& \facc{59.34}{2.76} \\
				\textbf{A3} w/o entropy loss & \facc{60.57}{1.27}\\
				\textbf{A4} {w/o semantic shift}& \facc{56.19}{1.25} \\
				\textbf{A5} {w/ uniform posterior}  & \facc{50.27}{1.04}  \\
				 \textbf{A2+A3}  & \facc{57.61}{1.72}  \\
				  \textbf{A2+A4}  & \facc{50.52}{2.01}  \\
				  	\textbf{A2+A5}  & \facc{48.59}{1.87}  \\
				  		 	\toprule
				  		   	\text{DGI}  & \facc{58.53}{2.98}  \\
				  	\text{DSSL}  & \bfacc{62.11}{1.53}  \\

				  \toprule[1.0pt]
			\end{tabular}
		}
	\end{wraptable}

\begin{figure*}[h]
\vskip -1em
	\centering
	\includegraphics[width=1.0\linewidth]{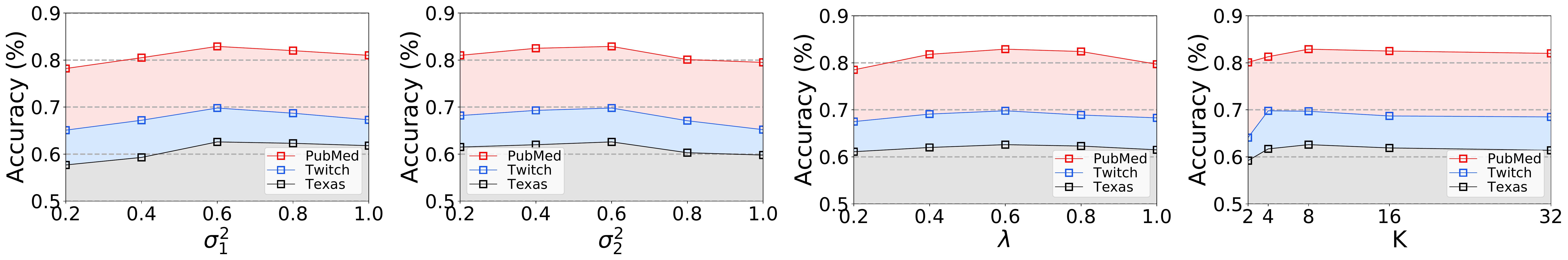}
	\vskip -1em
	\caption{Hyper-parameter analysis on PubMed, Twitch and Texas datasets.}\label{fig:parameters}
	\vskip -1em
\end{figure*}

\subsection{Ablation Study and Parameter Analysis}
\label{main:ab}
In this section, we conduct an ablation study and investigate the sensitivity of hyper-parameters.

\textbf{Ablation Study.}
We consider the following ablations: \textbf{(A1)} We remove the key component of DSSL: the local reconstruction loss (\texttt{w/o $\mathcal{L}_{local}$}). \textbf{(A2)} We remove global clustering loss  to see if $\mathcal{L}_{local}$ alone is still effective (\texttt{w/o $\mathcal{L}_{global}$}). \textbf{(A3)} We remove the entropy term in Eq.~\eqref{loss} (\texttt{w/o entropy}). \textbf{(A4)} We remove the semantic shift term $g_{\theta}(k)$ in  $\mathcal{L}_{local}$ (\texttt{w/o semantic shift}). \textbf{(A5)} We set posterior $q_{\phi}(k|\mathbf{v}_{i},\mathbf{z}_{j})={1}/{K}$ as uniform distribution for each link (\texttt{w/ uniform posterior}). We show the ablation study results in Table~\ref{tab_abl}. $\mathcal{L}_{global}$ alone does not provide much discriminative information, and it does not perform very well. $\mathcal{L}_{local}$ as a key component of DSSL alone produces better results than the DGI baseline.
We can also observe that considering semantic shift and personalized posterior can further improve the performance a lot which demonstrates our key motivations. The full model (last row) achieves the best performance, which illustrates that different components in the proposed DSSL are complementary to each other.

\textbf{Effect of Global Update on Prototypes}. 
We conduct ablation studies to gain insights into the global updating of prototypes. As shown in Figure~\ref{fig:curves}, we can observe that without Eq.~\eqref{Eq:global}, the performance is not very good, and we are stuck in bad local optima. As discussed above, a possible reason is that we will experience strong degeneracy if we only update prototype vectors with mini-batch training. This figure also demonstrates that DSSL can converge within a few hundred steps, which is efficient.

\textbf{Effect of Asymmetric Architecture}. 
We then explore the effect of target decay rate $\tau$ on the performance. Figure~\ref{fig:curves} shows the learning curves of DSSL on Squirrel and Twitch. We can observe that the best result is achieved at $\tau=0.9$. When $\tau=1.0$, i.e., the target network is never updated, DSSL obtains a competitive result but is lower than $\tau=0.9$. This confirms that slowly updating the target network is crucial in obtaining superior performance. At the other extreme value $\tau=0$, the target network is the same as the online network, and DSSL demonstrates a degenerated performance.

\textbf{Hyper-parameters Analysis.} 
We investigate the hyper-parameters most essential to our framework design, i.e., the standard deviation $\sigma_{1}$ and $\sigma_{2}$,  the temperature of the Gumbel Softmax $\gamma$, and the total number of factors $K$. The corresponding results are shown in Figure~\ref{fig:parameters}. 
$\sigma_1^{2}$ resembles the temperature scaling in Eq.~\eqref{loss}.
We observe $\sigma_1^{2}$
is better to be selected from $0.6$ to $1.0$, and
a too small (e.g., $0.2$) value may degenerate the performance in all three datasets. $\sigma_2^{2}$ balances the local and global loss in Eq.~\eqref{loss}. We can find that having large values of $\sigma_2^{2}$ does not improve the performance, as the local loss plays an essential role in the proposed model. Further, we observe that DSSL is not very sensitive to the Gumbel softmax temperature $\lambda$, while a moderate hardness of the Softmax gives the best results. We also find that as $K$ increases from $2$ to $8$, the performance of DGCL improves, which suggests the importance of decoupling latent factors. However, training with large $K$ will lead to a performance slightly drop. We provide results on other datasets in Appendix~\ref{app:ab}.

\begin{table}[t]
\centering
\caption{
Experimental results (\%)  on the node classification task with GAT. The best and second best performances under each dataset are marked with boldface and underline, respectively.}\label{table:gat}
\vspace{0.02in}
\scalebox{0.83}{
\begin{tabular}{ccccccccc}
\toprule
 \textbf{Method}       & \textbf{Cora}      & \textbf{Citeseer}    & \textbf{Pubmed}   & \textbf{Texas}    & \textbf{Cornell}      & \textbf{Squirrel}     & \textbf{Penn94} & \textbf{Twitch} \\
\midrule
MVGRL  &  \bfacc{84.52}{0.95}  &  \facc{71.52}{0.41} &  \underline{\facc{81.05}{0.68}} &  \underline{\facc{53.42}{0.29}} &  \underline{\facc{50.29}{0.96}}  &   \underline{\facc{37.58}{0.95}} &  \underline{\facc{57.21}{0.30}}  &   \underline{\facc{67.11}{0.23}} \\ 
BGRL  &  \facc{83.91}{0.25}  &  \underline{\facc{72.15}{0.42}} &  \facc{80.12}{0.52} &  \facc{51.02}{1.10} &  \facc{48.97}{1.03}  &   \facc{35.33}{1.12} &  \facc{56.52}{0.25}  &   \facc{66.21}{0.33} \\ 
DSSL  &  \facc{82.54}{0.46}  &   \bfacc{73.67}{0.68} &  \bfacc{81.69}{0.37} &   \bfacc{59.73}{1.28} &   \bfacc{53.63}{1.16}  &    \bfacc{39.42}{1.34} &   \bfacc{58.97}{0.26}  &   \bfacc{69.55}{0.48} \\ 
\bottomrule
\end{tabular}}
\vspace{-0.1in}
\end{table}

\textbf{Encoders Analysis.} 
To further evaluate the effectiveness of the proposed DSSL, we consider the case where the self-supervised learning methods are implemented using another GNN encoder. Table~\ref{table:gat} shows the results of the selected baseline with GAT as the encoder. As shown in Table~\ref{table:gat}, our DSSL performs better than all the compared methods in most cases, which once again proves the effectiveness of DSSL, and also shows that  DSSL is broadly applicable to various GNN encoders.

\begin{figure*}[t]
	\centering
	\includegraphics[width=1.01\linewidth]{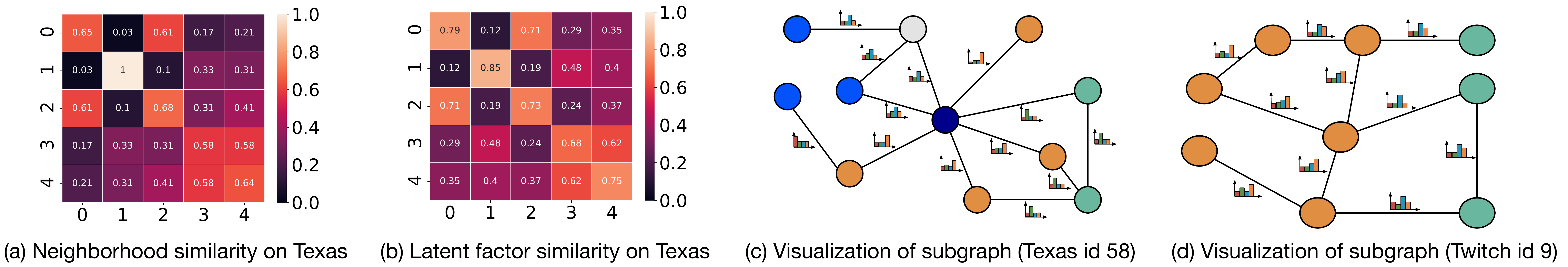}
	\vskip -0.5em
	\caption{The visualization and case study results (best
viewed on a computer screen and note that the latent distribution of each link need
to be zoomed in to be better visible).}\label{fig:visual}
\vskip -2em
\end{figure*}

\begin{figure*}[t]
	\centering
	\includegraphics[width=1.01\linewidth]{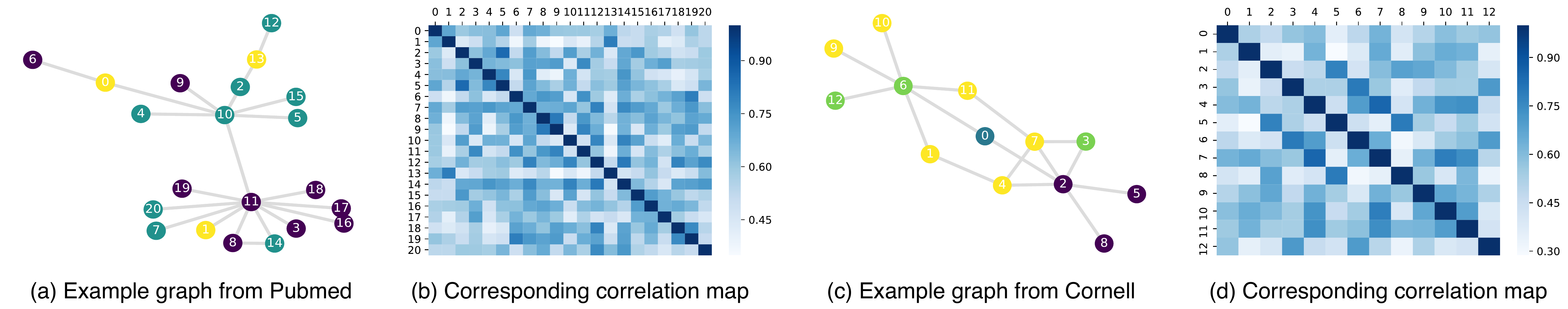}
	\vskip -0.5em
	\caption{Example graphs and correlation maps. The graphs are randomly sampled from Pubmed and Cornell. The correlation maps are obtained based on the pair-wise similarities of learned posteriors.}\label{fig:long-pub-cornell}
\vskip -2em
\end{figure*}

\subsection{Visualization and Case Study}\label{vis-main}
\textbf{Local neighborhood Patterns}. We provide the visualization and case study results to understand how DSSL uncovers the latent patterns in local neighborhoods. More results can be found in Appendix~\ref{app:vis-cs}. In Figure~\ref{fig:visual} (a), we calculate the cross-class neighborhood similarity (see Appendix~\ref{cross-neigh} for definition), which serves as the ground truth. If nodes with the same label share the same neighborhood distributions, the intra-class similarity should be high. In Figure~\ref{fig:visual} (b), we calculate average posterior distribution $q_{\phi}(k|\mathbf{v}_{i})$ over all nodes on each class and provide cosine similarity of them. We can observe that our learned distribution shares a similar pattern with the cross-class neighborhood similarity, which shows that learned latent factors can capture the latent semantic information related to neighborhood patterns. In Figures~\ref{fig:visual} (c) and (d), we plot the subgraph of a randomly selected  node, and the distribution $q_{\phi}(k|\mathbf{v}_{i},\mathbf{z}_{j})$.  We use different colors to indicate different labels. We find that similar neighbors over class generally have a similar distribution, while different types of links exhibit different latent distributions. This observation matches our motivation that  DSSL can decouple the diverse semantics in the local neighborhoods. 

\textbf{Global Semantic Patterns.}
We investigate how our learned semantic clusters perform on the long-range nodes. To show this, we study the learned cluster distribution $q_{\phi}(k|\mathbf{v}_{i})$ for each node. To facilitate visualization, we randomly sample a sub-graph that contains high-hop neighborhoods from each dataset and use different colors to indicate different node labels. We then calculate the pair-wise similarity of the posteriors $q_{\phi}(k|\mathbf{v}_{i})$ between nodes as shown in Figure~\ref{fig:long-pub-cornell}. We can observe that our learned semantic clusters exhibit similar patterns for the nodes in the same class. Some nodes exhibit similar semantic clusters regardless of the distance between them. For instance, node ID $6$ exhibits a very similar posterior distribution to the same class node ID $19$ in Pubmed, and node ID $3$ exhibits similar posterior distribution to node ID $12$ in Cornell, despite these nodes being ID $4$ hops away. Since our DSSL seeks to pull the node representation to the inferred semantic clusters, we can learn the global node-node relationships in representations instead of favoring nearby nodes.

\section{Conclusions}
\label{sec:conclusion}
In this paper, we study the problem of conducting self-supervised learning on non-homophilous graphs data. We present a novel decoupled self-supervised learning (DSSL) framework to decouple the diverse neighborhood context of a node in an unsupervised manner. Specifically, DSSL imitates the generative process of neighbors and explicitly models unobserved factors by latent variables. We show that DSSL can simultaneously capture the global clustering information and the local structure roles with the semantic shift. We theoretically show that  DSSL enjoys a good downstream performance. Extensive experiments on several graph datasets validated the superiority of DSSL and showed that DSSL could learn meaningful class-discriminative representations.

\section*{Acknowledgments}
This work was partially supported by the National Science Foundation (NSF) under grants number IIS-1909702 and IIS-1955851, Army Research Office (ARO) under grant number W911NF-21-1-0198 and Department of Homeland Security under grant number E205949D. The findings and conclusions in this paper do not necessarily reflect the view of the funding agency.

\bibliography{reference}
\bibliographystyle{plainnat}

\section*{Checklist}

\begin{enumerate}
\item For all authors...
\begin{enumerate}
 \item Do the main claims made in the abstract and introduction accurately reflect the paper's contributions and scope?
    \answerYes{}
 \item Did you describe the limitations of your work?
    \answerYes{}
 \item Did you discuss any potential negative societal impacts of your work?
    \answerYes{}
 \item Have you read the ethics review guidelines and ensured that your paper conforms to them?
    \answerYes{}
\end{enumerate}

\item If you are including theoretical results...
\begin{enumerate}
 \item Did you state the full set of assumptions of all theoretical results?
    \answerYes{}
        \item Did you include complete proofs of all theoretical results?
    \answerYes{}
\end{enumerate}

\item If you ran experiments...
\begin{enumerate}
 \item Did you include the code, data, and instructions needed to reproduce the main experimental results (either in the supplemental material or as a URL)?
      \answerYes{}
 \item Did you specify all the training details (e.g., data splits, hyperparameters, how they were chosen)?
    \answerYes{}
        \item Did you report error bars (e.g., with respect to the random seed after running experiments multiple times)?
    \answerYes{} 
        \item Did you include the total amount of compute and the type of resources used (e.g., type of GPUs, internal cluster, or cloud provider)?
        \answerNA{}
\end{enumerate}

\item If you are using existing assets (e.g., code, data, models) or curating/releasing new assets...
\begin{enumerate}
 \item If your work uses existing assets, did you cite the creators?
        \answerNA{}
 \item Did you mention the license of the assets?
        \answerNA{}
 \item Did you include any new assets either in the supplemental material or as a URL? \answerNA{}
 \item Did you discuss whether and how consent was obtained from people whose data you're using/curating?
        \answerNA{}
 \item Did you discuss whether the data you are using/curating contains personally identifiable information or offensive content?
        \answerNA{}
\end{enumerate}

\item If you used crowdsourcing or conducted research with human subjects...
\begin{enumerate}
 \item Did you include the full text of instructions given to participants and screenshots, if applicable?
    \answerNA{}
 \item Did you describe any potential participant risks, with links to Institutional Review Board (IRB) approvals, if applicable?
    \answerNA{}
 \item Did you include the estimated hourly wage paid to participants and the total amount spent on participant compensation?
    \answerNA{}
\end{enumerate}

\end{enumerate}

\newpage
\appendix
\onecolumn
\section{The Omitted Derivations}
\label{app:derivation}
\subsection{The Evidence Lower-Bound}
\label{app:ELBO}
In this section, we provide the details of the lower-bound in Eq.~\eqref{ELBO}. By introducing the approximated posterior $q_{\phi}(k|\mathbf{v}_{i},\mathbf{z}_{j})$, the likelihood $\mathcal{L}_{DSSL}(\theta,\boldsymbol{\mu})$ in Eq.~\eqref{generative} becomes:
\begin{linenomath}
\begin{align}
    \mathcal{L}_{DSSL}(\theta,\boldsymbol{\mu})&=\frac{1}{|{N(i)}|} \sum_{j\in N(i)} \log \Big[\sum_{k=1}^{K} p_{\theta}(\mathbf{z}_{j}|\mathbf{v}_{i},k)p_{\theta}(k|\mathbf{v}_{i})\Big] \nonumber \\
   & =\frac{1}{|{N(i)}|} \sum_{j\in N(i)} \log \Big[\sum_{k=1}^{K} p_{\theta}(\mathbf{z}_{j}|\mathbf{v}_{i},k)p_{\theta}(k|\mathbf{v}_{i})\frac{q_{\phi}(k|\mathbf{v}_{i},\mathbf{z}_{j})}{q_{\phi}(k|\mathbf{v}_{i},\mathbf{z}_{j})}\Big]  \\ 
    & \geq  \frac{1}{|{N(i)}|} \sum_{j\in N(i)}\sum_{k=1}^{K}{q_{\phi}(k| \mathbf{v}_{i}, \mathbf{z}_{j})}[\log p_{\theta}(\mathbf{z}_{j}| \mathbf{v}_{i}, k )+\log p_{\theta}(k| \mathbf{v}_{i})-\log q_{\phi}(k| \mathbf{v}_{i}, \mathbf{z}_{j})]  \nonumber \\
    & = \frac{1}{|{N(i)}|} \sum_{j\in N(i)} \mathbb{E}_{q_{\phi}(k| \mathbf{v}_{i}, \mathbf{z}_{j})}[\log p_{\theta}(\mathbf{z}_{j}| \mathbf{v}_{i}, k )+\log p_{\theta}(k| \mathbf{v}_{i})-\log q_{\phi}(k| \mathbf{v}_{i}, \mathbf{z}_{j})], \nonumber
\end{align}
\end{linenomath}
where the third step uses Jensen’s inequality. 

\subsection{The Optimization Loss}
\label{app:loss}
In this section, we derive the optimization loss for learning parameters in Eq.~\eqref{loss}.  Introducing the posterior  in Eq.~\eqref{Eq:posterior} and probabilities specification  $p_{\theta}(\mathbf{z}_{j}|\mathbf{v}_{i},k)=\mathcal{N}(\mathbf{z}_{j}; \boldsymbol{\mu}_{z_j},\boldsymbol{\Sigma}_{j})$  to the negative of the evidence lower-bounded in Eq.~\eqref{ELBO}, we have:
\begin{linenomath}
\begin{align}
\mathcal{L}=-\frac{1}{|{N(i)}|}& \sum_{j\in N(i)} \mathbb{E}_{q_{\phi}(k| \mathbf{v}_{i}, \mathbf{z}_{j})}[\log \mathcal{N}(\mathbf{z}_{j}; \boldsymbol{\mu}_{z_j},\boldsymbol{\Sigma}_{j})-\log \frac{ \mathcal{N}(\mathbf{v}_{i}; \boldsymbol{\mu}_{k},\mathbf{I}\sigma_{1}^{2})p(k)}{\sum_{k^{\prime}=1}^{K} \mathcal{N}(\mathbf{v}_{i}; \boldsymbol{\mu}_{k'},\mathbf{I}\sigma_{1}^{2})p(k')}\nonumber \\
&+\log q_{\phi}(k| \mathbf{v}_{i}, \mathbf{z}_{j})].\label{app:Eqloss}
\end{align}
\end{linenomath}
Since we consider isotropic Gaussian with $\boldsymbol{\Sigma}_{j}=\mathbf{I}\sigma_{2}^{2}$, we have:
\begin{linenomath}
\begin{align}
\mathcal{N}(\mathbf{z}_{j}; \boldsymbol{\mu}_{z_j},\boldsymbol{\Sigma}_{j})&=\frac{1}{(2 \pi)^{D' / 2}\sigma_2^{D'}} \exp (-\frac{1}{2\sigma_2^{2}}(\mathbf{z}_{j}-\mathbf{v}_{i}-\beta g_{\theta}(k))^{T}\cdot (\mathbf{z}_{j}-\mathbf{v}_{i}-\beta g_{\theta}(k))) \nonumber\\
&= \frac{1}{(2 \pi)^{D' / 2}\sigma_2^{D'}} \exp (-\frac{\|\mathbf{v}_{i}+\beta g_{\theta}(k)-\mathbf{z}_{j} \|_{2}^{2}}{2\sigma_{2}^{2}}) \label{app:loss-term1}
\end{align}
\end{linenomath}
Similarly, for $\mathcal{N}\left(\mathbf{v}_{i} ; \boldsymbol{\mu}_{k}, \mathbf{I} \sigma_{1}^{2}\right)$, we have:
\begin{linenomath}
\begin{align}
&\mathcal{N}(\mathbf{v}_{i} ; \boldsymbol{\mu}_{k}, \mathbf{I} \sigma_{1}^{2})=\frac{1}{(2 \pi)^{D' / 2}\sigma_{1}^{D'}} \exp (-\frac{1}{2\sigma_1^{2}}(\mathbf{v}_{i}-\boldsymbol{\mu}_{k})^{T}\cdot (\mathbf{v}_{i}-\boldsymbol{\mu}_{k}))  \\
&=\frac{1}{(2 \pi)^{D' / 2}\sigma_{1}^{D'}} \exp (-\frac{1}{2\sigma_1^{2}}(\mathbf{v}_{i}-\boldsymbol{\mu}_{k})^{T}\cdot (\mathbf{v}_{i}-\boldsymbol{\mu}_{k}))=\frac{1}{(2 \pi)^{D' / 2}\sigma_{1}^{D'}} \exp (\frac{(\mathbf{v}_{i}^{\top} \cdot \boldsymbol{\mu}_{k}-1)}{{\sigma}_{1}^{2}}),\nonumber \label{app:loss-term2}
\end{align}
\end{linenomath}
where the last equivalence holds because we apply  L2  normalization to  $\mathbf{v}_{i}$  and  $\boldsymbol{\mu}_{k}$, respectively. Substituting the corresponding terms in Eq.~\eqref{app:Eqloss} with Eq.~\eqref{app:loss-term1} and Eq.~\eqref{app:loss-term2}, and removing the constant  terms that are irrelevant to model parameters $\{\theta,\phi,\boldsymbol{\mu}\}$, we have:
\begingroup
\allowdisplaybreaks
\begin{linenomath}
\begin{align} 
&\mathcal{L}=\frac{1}{|{N(i)}|} \sum_{j\in N(i)}\frac{1}{2\sigma_2^{2}} \mathbb{E}_{q_{\phi}(k| \mathbf{v}_{i}, \mathbf{z}_{j})}\big[ \|\mathbf{v}_{i}+\beta g_{\theta}(k)-\mathbf{z}_{j}\|_{2}^{2}\big]-\mathcal{H}(q_{\phi}(k| \mathbf{v}_{i},\mathbf{z}_{j}))\nonumber
 \\
&-\frac{1}{|{N(i)}|} \sum_{j\in N(i)} \mathbb{E}_{q_{\phi}(k|\mathbf{v}_{i},\mathbf{z}_{j})}\Big[\log \frac{\exp \left(\boldsymbol{v}_{i}^{\top} \cdot \boldsymbol{\mu}_{k} / {\sigma}_{1}^{2}\right)}{\sum_{k^{\prime}=1}^{K}  \exp \left(\boldsymbol{v}_{i}^{\top} \cdot \boldsymbol{\mu}_{k’} / {\sigma}_{1}^{2}\right)}\Big]=\frac{1}{|{N(i)}|}\sum_{j\in N(i)}\frac{1}{2\sigma_2^{2}}\\
& \mathbb{E}_{q_{\phi}(k| \mathbf{v}_{i}, \mathbf{z}_{j})}\big[ \|\mathbf{v}_{i}+\beta g_{\theta}(k)-\mathbf{z}_{j}\|_{2}^{2}\big]- \mathbb{E}_{q_{\phi}(k|\mathbf{v}_{i})}[\log \frac{\exp (\boldsymbol{v}_{i}^{\top} \cdot \boldsymbol{\mu}_{k} / {\sigma}_{1}^{2})}{\sum_{k^{\prime}=1}^{K}  \exp (\boldsymbol{v}_{i}^{\top} \cdot \boldsymbol{\mu}_{k’} / {\sigma}_{1}^{2})}]-\mathcal{H}(q_{\phi}(k| \mathbf{v}_{i},\mathbf{z}_{j}))\nonumber
\end{align}
\end{linenomath}
\endgroup%
where $q_{\phi}(k|\mathbf{v}_{i})={1}/{|{N(i)}|} \sum\nolimits_{j\in N(i)} q_{\phi}(k| \mathbf{v}_{i}, \mathbf{z}_{j})$. Multiplying the constant $2\sigma_2^{2}$ for all terms and absorbing  $2\sigma_2^{2}$ to the hyper-parameter $\sigma_2^{2}$, we obtain the loss function in Eq.~\eqref{loss}. 

\subsection{Derivation of Global Update}
\label{app:global}
In this section,  we derive the global analytical update of prototypes. Specifically, we  have:
\begin{linenomath}
\begin{align} 
\boldsymbol{\mu}_{k} \leftarrow \arg \min_{\boldsymbol{\mu}_{k}} \sum_{i=1}^{N} \mathcal{L}&=\arg \max_{\boldsymbol{\mu}_{k}} \sum_{i=1}^{N} \mathbb{E}_{q_{\phi}(k|\mathbf{v}_{i})}\Big[\log \frac{\exp \left(\boldsymbol{v}_{i}^{\top} \cdot \boldsymbol{\mu}_{k} / {\sigma}_{1}^{2}\right)}{\sum_{k^{\prime}=1}^{K}  \exp \left(\boldsymbol{v}_{i}^{\top} \cdot \boldsymbol{\mu}_{k’} / {\sigma}_{1}^{2}\right)}\Big] \nonumber \\
&=\arg \max_{\boldsymbol{\mu}_{k}}\sum_{i=1}^{N} \mathbb{E}_{q_{\phi}(k|\mathbf{v}_{i})} \Big[\log \exp (\boldsymbol{v}_{i}^{\top} \cdot \boldsymbol{\mu}_{k} / {\sigma}_{1}^{2}) \Big] \nonumber \\
&=\arg \max_{\boldsymbol{\mu}_{k}}\sum_{i=1}^{N} \mathbb{E}_{q_{\phi}(k|\mathbf{v}_{i})} \Big[ (\boldsymbol{v}_{i}^{\top} \cdot \boldsymbol{\mu}_{k} / {\sigma}_{1}^{2}) \Big].
\end{align}
\end{linenomath}
In the second line, following~\cite{wu2018unsupervised}, we treat the normalizing term as the constant. If we  apply L2 normalization to $\boldsymbol{\mu}_{k}$: $\|\boldsymbol{\mu}_{k}\|=1$, we need to solve the Lagrangian of the objective function:
\begin{linenomath}
\begin{align} 
\mathcal{L}(\boldsymbol{\mu}_{k},\beta)=\max_{\boldsymbol{\mu_{k}},\beta} \sum_{i=1}^{N} \mathbb{E}_{q_{\phi}(k|\mathbf{v}_{i})} \Big[ (\mathbf{v}_{i}^{\top} \cdot \boldsymbol{\mu}_{k} / {\sigma}_{1}^{2}) \Big]+\beta (1-\boldsymbol{\mu}_{k}^{\top} \boldsymbol{\mu}_{k}).
\end{align}
\end{linenomath}
Take the gradient over $\boldsymbol{\mu}_{k}$ with respect to $\mathcal{L}(\boldsymbol{\mu}_{k},\beta)$ and set it to zero, we have:
\begin{linenomath}
\begin{align} 
&\sum_{i=1}^{N} {q_{\phi}(k|\mathbf{v}_{i})}\mathbf{v}_{i}/ {\sigma}_{1}^{2}-2 \beta \cdot \boldsymbol{\mu}_{k}=0 \nonumber \\
&\Rightarrow \boldsymbol{\mu}_{k}=\frac{\sum_{i=1}^{N}q_{\phi}(k|\mathbf{v}_{i})\cdot\mathbf{v}_{i}}{2\beta \sigma_1^{2}}=\frac{\sum_{i=1}^{N}\pi_{i}(k) \cdot \mathbf{v}_{i}}{2\beta \sigma_1^{2}}, \label{app:eqlagrange-mu}
\end{align}
\end{linenomath}
where $\pi_{i}(k)=q_{\phi}(k|\mathbf{v}_{i})$. By taking the gradient with respect to multiplier $\beta$ and setting it to zero, we have $\boldsymbol{\mu}_{k}^{\top}\boldsymbol{\mu}_{k}=1$. Combining $\boldsymbol{\mu}_{k}^{\top}\boldsymbol{\mu}_{k}=1$ and Eq.~\eqref{app:eqlagrange-mu}, we have:
\begin{linenomath}
\begin{align} 
  \beta=\frac{\|\sum_{i=1}^{N}\pi_{i}(k)\cdot \mathbf{v}_{i}\|}{2\sigma_{1}^{2}}.
\end{align}
\end{linenomath}
Further combining the above equation with  Eq.~\eqref{app:eqlagrange-mu}, we obtain the analytical global update:
\begin{linenomath}
\begin{align} 
  \boldsymbol{\mu}_{k}=\frac{\sum_{i=1}^{N}\pi_{i} (k)\cdot \mathbf{v}_{i}}{\|\sum_{i=1}^{N}\pi_{i} (k)\cdot \mathbf{v}_{i} \|_{2}^{2}},~~ \text{where}~~ \pi_{i} (k)={1}/{|{N(i)}|} \sum\nolimits_{j\in N(i)} q_{\phi}(k| \mathbf{v}_{i}, \mathbf{z}_{j}).
\end{align}
\end{linenomath}
\begin{algorithm}[H]
				\caption{Training Algorithm of DSSL}\label{app:alg1}
				\textbf{Input:}\hspace{0mm} $G=(\mathcal{V}, \mathcal{E})$ \\
				\textbf{Output:}\hspace{0mm} Online encoder network  parameters $\theta$.\\
				Initialize target encoder parameter $\xi=\theta$\\
				\Repeat{convergence or reaching max iteration}{
					Randomly select a mini-batch nodes from  $\mathcal{V}$\\
					\For{each $v_i$ in the batch}{
						Randomly sample its neighbors  $\mathcal{N}(i)$\\
						Sample $\mathbf{c}$ with Gumbel softmax
					}
					$\mathcal{L}\leftarrow$ Eq.~\eqref{loss}\\
				$[ \theta, \phi, \boldsymbol{\mu}] \leftarrow  [ \theta, \phi, \boldsymbol{\mu}]-\boldsymbol{\Gamma}\left(\nabla_{ \theta, \phi, \boldsymbol{\mu}} \mathcal{L}\right)$\\
				Update $\xi$ with momentum moving average: $\xi \leftarrow \tau \xi+(1-\tau) \theta$	\\
				Update $\boldsymbol{\mu}$ at the end of each training epoch with Eq.~\eqref{Eq:global}
				}
			\end{algorithm}

\section{The Training Algorithm and Network Architecture}
\label{app:alg}
\subsection{The Training Algorithm of DSSL}
The overall training algorithm is shown in Algorithm ~\ref{app:alg1}.   Concretely, for each iteration of DSSL, we randomly sample the nodes and their neighbors. Then, we calculate the loss function with Gumbel softmax approximation. All online trainable parameters are updated with stochastic gradient descent (SGD), while those marked with $\xi$ are the target network counterparts to be updated with exponential moving average. We  apply a global update for the  prototypes $\boldsymbol{\mu}$ at the end of each training epoch.
\begin{figure}
\centering
    \includegraphics[width=0.85\textwidth]{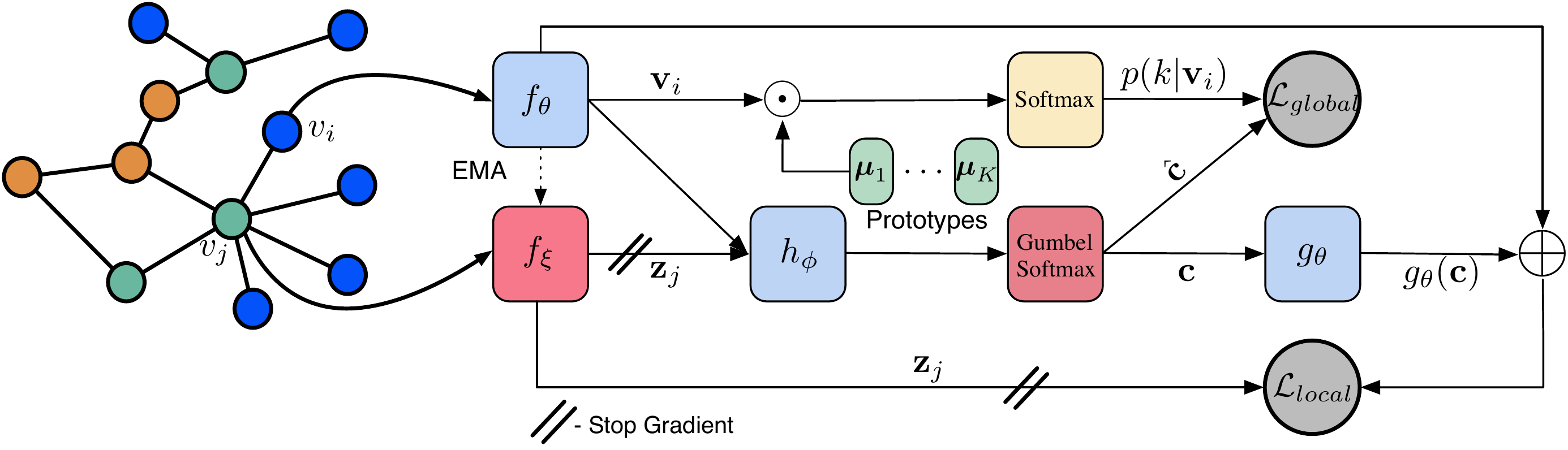}
    \vskip -0.5em
     \caption{An illustration of the overall DSSL framework. For simplicity, we only illustrate the learning process on one neighbor $v_{j}$ of central node $v_{i}$ while it can be applied to all neighbors.}\label{fig:network}
\end{figure}
 \vskip -1.5em
\subsection{The Network Architecture of DSSL}
The DSSL framework is illustrated in Figure~\ref{fig:network}. DSSL proposes a self-supervised scheme that learns to encode the node using an online encoder along with a set of prototypes, inference projector $g_{\theta}$, and predictor $h_{\phi}$.  A target encoder is used to produce presentations of neighbors and updated by exponential moving average. We stop gradients on the target encoder coming from the local loss to introduce asymmetry between the online and target encoders.

\section{{Efficiency Analysis}}
\subsection{{Time Complexity Analysis}}
\label{sec:Time}
We give the detailed time complexity per iteration of  DSSL and compare it with representative methods on node representation learning such as GCA~\cite{zhu2021graph} and BGRL~\cite{thakoor2021large}. For  self-supervised learning, the time complexity mainly depends on two parts: the encoder and loss calculation. Here, we assume the backward pass to be approximately as costly as a forward pass following~\cite{thakoor2021large}.
Since our DSSL is agnostic to GNN encoders, we consider $L$-layer GCN as an example. A $L$-layer GCN~\cite{DBLP:conf/iclr/KipfW17} with $d$ hidden dimensions has $\mathcal{O}\left(d L|\mathcal{E}|+Nd^{2} L\right)$ complexity where $\mathcal{E}$ is the number of edges, since it costs $\mathcal{O}(d|\mathcal{E}|)$ to propagate message in each layer, and $\mathcal{O}\left(N d^{2}\right)$ to multiply by the weight matrix in each layer. Our DSSL additionally has the inference predictor and projector components in the forward pass, which costs $\mathcal{O}(C_{\text {predictor}}N+C_{\text {projector}}N)$ where $C_{\cdot}$ are constants depending on architecture of the different components.  The complexity of DSSL on loss calculation is $\mathcal{O}(KNd)+\mathcal{O}(|\mathcal{E}|d)$ where $\mathcal{O}(KNd)$ and $\mathcal{O}(|\mathcal{E}|d)$  are for global and local losses, respectively. Thus the total time  complexity per iteration for DSSL with GCN is $\mathcal{O}\left(d L|\mathcal{E}|+Nd^{2} L+C_{\text {predictor}}N+C_{\text {projector}}N+KNd+|\mathcal{E}|d\right)$. With the similar notations, we can express the time complexity of GCA as $\mathcal{O}\left(d L|\mathcal{E}|+Nd^{2} L+C_{\text {projector}}N+N^{2}d\right)$ and that of  BGRL as $\mathcal{O}\left(d L|\mathcal{E}|+Nd^{2} L+C_{\text {predictor}}N+Nd\right)$. Typically, the architectures for predictor and projector are one-layer MLP for all methods. Thus the $C_{\text {predictor}}$ and $C_{\text {projector}}$ in different methods should be similar. Based on the analysis above, we can find that the time complexity of GCA scales quadratically in the size of the nodes, and both BGRL and our DSSL linearly increase with the number of nodes $N$ when $\mathcal{E}$ is proportional to $N$. Thus the overall time complexity of our DSSL is at the same level as BGRL~\cite{thakoor2021large}, which is a large-scale self-supervised learning method proposed recently.

\section{Theoretical Analysis}
\label{app:theory}
\subsection{Proof of Theorem~\ref{the:mutual}}
To prove Theorem~\ref{the:mutual}, we first present three Lemmas:
\begin{lemma}\label{Lemma-1}
For any random variables $X$, $Y$, $Z$ and $V$, we have the following relations:
\begin{align}
    I(X, Y ; Z|V) \geq I(X ; Z|V)
\end{align}
\end{lemma}
\begin{proof}

\begingroup
\allowdisplaybreaks
\begin{linenomath}
\begin{align}
&I(X, Y ; Z|V)-I(X ; Z|V) \nonumber \\
&=\iiiint_{X Y Z V} p(V) p(X, Y, Z|V) \log \frac{p(X, Y, Z|V)}{p(X, Y|V) p(Z|V)} d X d Y d Z dV \nonumber \\
&- \iiint_{X Z V} p(V)p(X, Z|V) \log \frac{p(X, Z|V)}{p(X|V) p(Z|V)} d X d ZdV \nonumber =\iiiint_{X Y Z V}P(V) p(X, Y, Z|V)\\
&  \log \frac{p(X, Y, Z|V)}{p(X, Y|V) p(Z|V)}- p(V) p(X, Y, Z|V) \log \frac{p(X, Z|V)}{p(X|V) p(Z|V)} d X d Y d Z dV\nonumber \\
&=\iiint_{X Y Z V} p(V) p(X, Y, Z|V) \log \frac{p(X, Y, Z|V)}{p(Y |X,V) p(X, Z|V)} d X d Y d Z dV = \\
&\iiint_{X Y Z V} p(V)p(Y, Z|X,V) p(X|V) \log \frac{p(Y, Z| X,V)}{p(Y| X,V) p(Z |X,V)} d X d Y d Z dV=I(Y;Z|X,V)\geq 0, \nonumber
\end{align}
\end{linenomath}
\endgroup
which completes the proof of  this Lemma.
\end{proof}

\begin{lemma}\label{Lemma-2}
For any random variables $X$ and $Z$, the mutual information $I(X; Z)$ has the following variational lower bound:
\begin{linenomath}
\begin{align}
  I(X;Z)=H(Z)-H(Z|X)&=\mathbb{E}_{q(Z)}[\log q(Z)]+\mathbb{E}_{q(Z)q(X|Z)}[\log q(Z|X)] \label{Eq:lemma2}\\
  &\geq \max_{\theta} \mathbb{E}_{q(Z)}[\log q(Z)]+\mathbb{E}_{q(Z)q(X|Z)}[\log p_{\theta}(Z|X)], \nonumber
\end{align}
\end{linenomath}
where $p_{\theta}(Z|X)$ is an introduced variational discriminator with $\theta$ representing the parameters.\end{lemma}
\begin{proof}
\begin{linenomath}
\begin{align}
\mathbb{E}_{q(Z)q(Z|X)}[\log q(Z|X)]&=\max_{\theta}\mathbb{E}_{q(Z)q(X|Z)}[\log p_{\theta}(Z|X)]+KL(q(Z|X)||p_{\theta}(Z|X)) \nonumber \\
&\geq \max_{\theta}\mathbb{E}_{q(Z)q(X|Z)}[\log p_{\theta}(Z|X)].
\end{align}
\end{linenomath}
Plugging the inequality above into Eq.~\eqref{Eq:lemma2} completes the proof. 
\end{proof}
\begin{lemma}\label{Lemma-3}
For any random variables $X$  and $Z$, the mutual information $I(X, Z)$ has the following contrastive lower bound:
    \begin{linenomath}
\begin{align} 
I(X;Z) \geq \mathbb{E}_{p(X, Z)}[f_{{\theta}}(x, z)-\mathbb{E}_{q(\tilde{\mathcal{B}})}[\log \sum_{\tilde{z} \in \tilde{\mathcal{B}}} \exp f_{\boldsymbol{\theta}}(x, \tilde{z})]]+\log |\tilde{\mathcal{B}}|
    \end{align}
\end{linenomath}
where $x$ and $z$ are the sample instances, $f_{{\theta}} \in \mathbb{R}$ is a function parameterized by ${\theta}$ (e.g., a dot product between encoded representations of $x$ and $z$), and $\tilde{\mathcal{B}}$ is a set of samples drawn from a proposal distribution $q(\tilde{\mathcal{B}})$. The set $\tilde{\mathcal{B}}$ contains the one positive sample  and $|\tilde{\mathcal{B}}|-1$ negative samples.
\end{lemma}
\begin{proof}
The  proof of this contrastive estimation bound can be founded in~\cite{poole2019variational}.
\end{proof}
Now, we are ready to prove our Theorem~\ref{the:mutual}. We first restate Theorem~\ref{the:mutual}:

\textbf{Theorem~\ref{the:mutual}}.
    \textit{Optimizing local and global terms in Eq.~\eqref{loss} is equivalent to maximizing the mutual information between the representation $\mathbf{v}$ and global signal ${k}$ and maximizing the conditional mutual information between $\mathbf{v}$ and the local signal  $\mathbf{z}$, conditioned on global signal $k$. Formally, we have:}
    \begin{linenomath}
\begin{align} 
       \max_{\theta, \phi, \boldsymbol{\mu}} \mathcal{L}\Rightarrow  \max_{\mathbf{v}} {I}(\mathbf{v};k)+ {I}(\mathbf{v};\mathbf{z}|k)= {I}(\mathbf{v};k,\mathbf{z}).
    \end{align}
\end{linenomath}
\begin{proof}
According to Lemma~\ref{Lemma-1}, for any $v_j \in N(v)$ we have:
\begin{align}
I(\mathbf{v};\mathbf{z}|k)=I(\mathbf{v};\mathbf{z}_{1},\cdots, \mathbf{z}_{N(v)}|k)\geq I (\mathbf{v};\mathbf{z}_{j}|k).
\end{align}
Given the above, we further have the following:
\begin{linenomath}
\begin{align} 
I(\mathbf{v};\mathbf{z}|k)=\frac{1}{|\mathcal{N}(v)|}\sum_{j\in N(v)}I(\mathbf{v};\mathbf{z}|k)\geq \frac{1}{|\mathcal{N}(v)|}\sum_{j\in N(v)} I (\mathbf{v};\mathbf{z}_{j}|k). \label{appEq:sum}
\end{align}
\end{linenomath}
For $I (\mathbf{v};\mathbf{z}_{j}|k)$, we can relate it to our local loss $\mathcal{L}_{local}$ in the main body of the paper:
\begin{linenomath}
\begin{align} 
\max_{\mathbf{v}} I (\mathbf{v};\mathbf{z}_{j}|k)&\Leftrightarrow \max_{\mathbf{v}} H(\mathbf{z}_{j}|k)-H(\mathbf{z}_{j}|k,\mathbf{v}) \nonumber \\
&\Leftarrow \min_{\mathbf{v},\theta}- \mathbb{E}_{q(\mathbf{v})q(\mathbf{z}_{j})q(k|\mathbf{v},\mathbf{z}_{j})}[\log p_{\theta}(\mathbf{z}_{j}|k,\mathbf{v})]\nonumber\\
&\Leftrightarrow \min_{\theta,\phi}- \mathbb{E}_{q_{\theta}(\mathbf{v})q_{\xi}(\mathbf{z}_{j})q_{\phi}(k|\mathbf{v},\mathbf{z}_{j})}[\log p_{\theta}(\mathbf{z}_{j}|k,\mathbf{v})]\nonumber\\
&\Leftrightarrow \min_{\theta,\phi}- \mathbb{E}_{\mathbf{v}=f_{\theta}(\mathcal{G})[v], \mathbf{z}_{j}=f_{\xi}(\mathcal{G})[j],q_{\phi}(k|\mathbf{v},\mathbf{z}_{j})}[\log p_{\theta}(\mathbf{z}_{j}|k,\mathbf{v})] \nonumber \\
&\Leftrightarrow \min_{\theta,\phi}- \mathbb{E}_{\mathbf{v}=f_{\theta}(\mathcal{G})[v], \mathbf{z}_{j}=f_{\xi}(\mathcal{G})[j],q_{\phi}(k|\mathbf{v},\mathbf{z}_{j})}[\log p_{\theta}(\mathbf{z}_{j}|k,\mathbf{v})] \nonumber \\
&\Leftrightarrow  \min_{\theta,\phi} \frac{1}{2 \sigma_{2}^{2}} \mathbb{E}_{q_{\phi}(k \mid \mathbf{v}, \mathbf{z}_{j})}[\left\|\mathbf{v}+\beta g_{\theta}(k)-\mathbf{z}_{j}\right\|_{2}^{2}], \label{appEq:local}
\end{align}
\end{linenomath}
where we use Lemma~\ref{Lemma-2} in the second line, and $\mathbf{v}$ and $\mathbf{z}_{j}$ are the representations of node $v$ and node $z_j$ through deterministic encoder $f$ with input $\mathcal{G}=\{\mathbf{X},\mathbf{A}\}$. Note that the target encoder $\xi$ is  updated via the exponential moving average of the online encoder $\theta$. Combining Eqs.~\eqref{appEq:local} and ~\eqref{appEq:sum}, we have:
\begin{linenomath}
\begin{align} 
    \min_{\theta,\phi} \frac{1}{\mathcal{N}(v)} \sum_{j\in\mathcal{N}(v)} \frac{1}{2 \sigma_{2}^{2}} \mathbb{E}_{q_{\phi}(k \mid \mathbf{v}, \mathbf{z}_{j})}[\left\|\mathbf{v}+\beta g_{\theta}(k)-\mathbf{z}_{j}\right\|_{2}^{2}] \Rightarrow \max_{\mathbf{v}}I(\mathbf{v};\mathbf{z}|k).
\end{align}
\end{linenomath}
Our global loss is further related to contrastive estimation. Using Lemma~\ref{Lemma-3} on $I(\mathbf{v};k)$, we have:
\begin{linenomath}
\begin{align} 
I(\mathbf{v};k)\geq \mathbb{E}_{p(\mathbf{v},k)}[f_{{\theta}}(\mathbf{v}, k )-\mathbb{E}_{q(\tilde{\mathcal{B}})}[\log \sum_{\tilde{k} \in \tilde{\mathcal{B}}} \exp f_{\boldsymbol{\theta}}(\mathbf{v}, \tilde{k})]+\log |\tilde{\mathcal{B}}|.
\end{align}
\end{linenomath}
When $\tilde{\mathcal{B}}$ always includes all possible values of the latent  variable $k$ (i.e., $|\tilde{\mathcal{B}}|=K$) and they are uniformly distributed, and setting $f_{{\theta}}(\mathbf{v}, k )={\mathbf{v}^{\top} \cdot \boldsymbol{\mu}_{k}}/{\sigma_{1}^{2}}$, we have:
\begin{linenomath}
\begin{align} 
I(\mathbf{v};k)\geq \mathbb{E}_{\mathbf{v}=f_{\theta}(\mathcal{G})[v],q_{\phi}(k|\mathbf{v}_{})}[\log \frac{\exp \left(\mathbf{v}_{}^{\top} \cdot \boldsymbol{\mu}_{k} / {\sigma}_{1}^{2}\right)}{\sum_{k^{\prime}=1}^{K}  \exp (\mathbf{v}_{}^{\top} \cdot \boldsymbol{\mu}_{k’} / {\sigma}_{1}^{2})}].
\end{align}
\end{linenomath}
Given the above, we have:
\begin{linenomath}
\begin{align} 
    \max_{\theta,\phi,\boldsymbol{\mu}}\mathbb{E}_{\mathbf{v}=f_{\theta}(\mathcal{G})[v],q_{\phi}(k|\mathbf{v}_{})}[\log \frac{\exp \left(\mathbf{v}_{}^{\top} \cdot \boldsymbol{\mu}_{k} / {\sigma}_{1}^{2}\right)}{\sum_{k^{\prime}=1}^{K}  \exp \left(\mathbf{v}_{}^{\top} \cdot \boldsymbol{\mu}_{k’} / {\sigma}_{1}^{2}\right)}] \Rightarrow \max_{\mathbf{v}} I(\mathbf{v};k). \label{appEq:global}
\end{align}
\end{linenomath}
Combining Eqs.~\eqref{appEq:local} and ~\eqref{appEq:global}, and
omitting the entropy term in Eq.~\eqref{loss} completes the proof.
\end{proof}

\subsection{Proof of Theorem~\ref{the:bound}}
To prove Theorem~\ref{the:bound}, we first present the following Lemmas.
\begin{lemma}\label{Lemma-4}
For a representation $\mathbf{v}$ that is obtained with a deterministic GNN encoder $f_{\theta}$ of input graph $\mathcal{G}$ with enough capacity, we have the data processing Markov chain: $(k,\mathbf{z})\leftrightarrow \mathbf{y} \leftrightarrow \mathcal{G} \rightarrow \mathbf{v}$.
\end{lemma}
\begin{proof}
Since $\mathbf{v}=f_{\theta}(\mathcal{G})$ is a deterministic function of input graph $\mathcal{G}$, we have the following conditional independence: $(k,\mathbf{z}) \perp \mathbf{v}|\mathcal{G}$ and $\mathbf{y} \perp \mathbf{v}|\mathcal{G}$~\cite{federici2019learning}, which leads to the data processing Markov chain $(k,\mathbf{z})\leftrightarrow \mathbf{y} \leftrightarrow \mathcal{G} \rightarrow \mathbf{v}$. Thus, the proof is completed.
\end{proof}

\begin{lemma}
By optimizing local and global terms  in Eq.~\eqref{loss},  $\mathbf{v}_{joint}$ is both minimal and sufficient: $\mathbf{v}_{joint}$ is sufficient $\mathbf{v}_{joint}=\arg\max_{\mathbf{v}}I(\mathbf{v};k,\mathbf{z})$; $\mathbf{v}_{joint}$ is minimal $\mathbf{v}_{joint}=\arg\min_{\mathbf{v}}H(\mathbf{v}|k,\mathbf{z})$.
\end{lemma}
\begin{proof}
We have $I(\mathbf{v};\mathbf{z}|k)=H(\mathbf{v}|k)-H(\mathbf{v}|k,\mathbf{z})$, in which the entropy $H(\mathbf{v}|k)$ is a constant since $p(\mathbf{v}|k)$ is a Gaussian distribution. Thus, maximizing $I(\mathbf{v};\mathbf{z}|k)$ is equivalent to minimize $H(\mathbf{v}|k,\mathbf{z})$. Combining this with Theorem~\ref{the:mutual}, one can directly conclude that $\mathbf{v}_{joint}$ is sufficient and minimal.
\end{proof}

Next, we restate Theorem~\ref{the:bound} and provide a complete proof with the help of the Lemmas above. 

\textbf{Theorem~\ref{the:bound}}. 
\textit{Let $\mathbf{v}_{\mathrm{joint}}=\arg \max_{\mathbf{v}} I(\mathbf{v};\mathbf{z},k),\mathbf{v}_{\mathrm{local}}=\arg \max_{\mathbf{v}} I(\mathbf{v};k)$, and $\mathbf{v}_{\mathrm{global}}=\arg \max_{\mathbf{v}}I(\mathbf{v};\mathbf{z})$. Formally, we have the following inequalities about the task-relevant information:}
   \begin{align}
        I(\mathcal{G};\mathbf{y})=\max_{\mathbf{v}} I(\mathbf{v};\mathbf{y})\geq I(\mathbf{v}_{\mathrm{joint}};\mathbf{y})\geq \max (I(\mathbf{v}_{\mathrm{local}};\mathbf{y}),I(\mathbf{v}_{\mathrm{global}};\mathbf{y}))\geq I(\mathcal{G};\mathbf{y})-\epsilon. \label{Eq:proposition}
   \end{align}
\begin{proof}
Based on Lemma~\ref{Lemma-4}, we have the following data processing inequality~\cite{thomas2006elements}:
\begin{align}
    I(k,\mathbf{z};\mathcal{G}) \geq I(k,\mathbf{z};\mathbf{v}),  I(k,\mathbf{z};\mathcal{G}; \mathbf{y}) \geq I(k,\mathbf{z};\mathbf{v}; \mathbf{y}), I(\mathcal{G},\mathbf{y})\geq I(\mathbf{v},\mathbf{y}).
\end{align}
Since $\mathbf{v}_{\mathrm{joint}}=\arg \max_{\mathbf{v}} I(\mathbf{z},k;\mathbf{v})$, we  have: $I(k,\mathbf{z};\mathbf{v}_{\mathrm{joint}})=I(k,\mathbf{z};\mathcal{G})$ and $I(k,\mathbf{z};\mathbf{v}_{\mathrm{joint}};\mathbf{y})=I(k,\mathbf{z};\mathcal{G};\mathbf{y})$. In addition, since $\mathbf{v}_{joint}$ is minimal, we also have, $I(\mathbf{v}_{joint} ; \mathbf{y}| k,\mathbf{z}) \leq H(\mathbf{v}_{joint}|k,\mathbf{z})=0$.
Give the above, we have the following equality:
\begin{linenomath}
\begin{align}
I(\mathbf{v}_{joint}; \mathbf{y}) &=I(\mathbf{v}_{joint}; \mathbf{y} ; k,\mathbf{z})+I(\mathbf{v}_{joint} ; \mathbf{y} |k,\mathbf{z}) \nonumber \\
&=I(\mathcal{G} ; \mathbf{y} ; k,\mathbf{z})+I(\mathbf{v}_{joint} ; \mathbf{y} |k,\mathbf{z}) \nonumber \\
&=I(\mathcal{G} ; \mathbf{y} ; k,\mathbf{z})+0\nonumber \\
&=I(\mathcal{G} ; \mathbf{y})-I(\mathcal{G} ; \mathbf{y} | k,\mathbf{z}) \nonumber \\
&=\max_{\mathbf{v}}I(\mathbf{v} ; \mathbf{y}) -I(\mathcal{G} ; \mathbf{y} | k,\mathbf{z})=I(\mathbf{v}_{\mathrm{sup}} ; \mathbf{y}) -I(\mathcal{G} ; \mathbf{y} | k,\mathbf{z}).\label{appEq:mutual-bound}
\end{align}
\end{linenomath}
Thus, the mutual information gap  between self-supervised representation $\mathbf{v}_{joint}$ and supervised representation $\mathbf{v}_{sup}$ is $I(\mathcal{G} ; \mathbf{y} | k,\mathbf{z})$. Based on the property of mutual information, we further have:
\begin{linenomath}
\begin{align}
I(\mathcal{G} ; \mathbf{y}|k)=I(\mathcal{G} ; \mathbf{y};\mathbf{z}|k)+I(\mathcal{G} ; \mathbf{y} | k,\mathbf{z})\geq I(\mathcal{G} ; \mathbf{y} | k,\mathbf{z}).\label{appEq:mutual-bound2}
\end{align}
\end{linenomath}
Similarly, we have $I(\mathcal{G}; \mathbf{y}|\mathbf{z})\geq I(\mathcal{G} ; \mathbf{y} | k,\mathbf{z})$. Combining Eqs.~\eqref{appEq:mutual-bound} and ~\eqref{appEq:mutual-bound2},  we have the following inequalities about the task-relevant information  based on our Assumption~\ref{assumption}:
\begin{linenomath}
   \begin{align}
    I(\mathcal{G};\mathbf{y})=\max_{\mathbf{v}} I(\mathbf{v};\mathbf{y})\geq I(\mathbf{v}_{\mathrm{joint}};\mathbf{y})\geq \max (I(\mathbf{v}_{\mathrm{local}};\mathbf{y}),I(\mathbf{v}_{\mathrm{global}};\mathbf{y}))\geq I(\mathcal{G};\mathbf{y})-\epsilon,
   \end{align}
   \end{linenomath}
which completes the proof
\end{proof}

\subsection{Proof of Corollary~\ref{the:bayes}}
Restate Corollary~\ref{the:bayes}:

\textbf{Corollary~\ref{the:bayes}.} \textit{Suppose that downstream label $\mathbf{y}$ is a M-categorical random variable. Then we have the
upper bound for the downstream Bayes errors ${P}^{e}_{\mathbf{v}}=\mathbb{E}_{\mathbf{v}}\left[1-\max _{y \in \mathbf{y}} P(\hat{\mathbf{y}}=y| \mathbf{v})\right]$ on learned representation $\mathbf{v}$, where $\hat{\mathbf{y}}$ is the estimation for label from our downstream classifier:}
\begin{linenomath}
   \begin{align}
    \operatorname{Th}({P}_{\mathbf{v}_{joint}}^{e}) \leq \log 2+P_{\mathbf{v }_{\mathrm{sup}}}^{e} \cdot \log M+I(\mathcal{G} ; \mathbf{y}|\mathbf{z},k) \triangleq {\mathrm{RHS}}_{\mathbf{v}_{joint}}
\end{align}
   \end{linenomath}
\textit{where $\operatorname{Th}(x)=\min \{\max \{x, 0\}, 1-1 /|M|\}$ is a thresholded operation~\cite{tsai2020demystifying}. Similarly, we can obtain the  error upper bound of other representations $\mathbf{v}_{\mathrm{local}}$ and $\mathbf{v}_{\mathrm{global}}:$ ${\mathrm{RHS}}_{\mathbf{v}_{local}}$and ${\mathrm{RHS}}_{\mathbf{v}_{global}}$. Then, we have   inequalities on error upper bounds $: {\mathrm{RHS}}_{\mathbf{v}_{joint}}\leq \min ({\mathrm{RHS}}_{\mathbf{v}_{local}},{\mathrm{RHS}}_{\mathbf{v}_{global}})$.}
\begin{proof}
To prove this Corollary,  we use the following inequalities~\cite{thomas2006elements,tsai2020demystifying}:
\begin{linenomath}
   \begin{align}
& \operatorname{Th}({P}_{\mathbf{v}_{joint}}^{e}) \leq-\log (1-{P}_{\mathbf{v}_{joint}}^{e}) \leq H(\mathbf{y}| \mathbf{v}_{\mathrm{joint}}) \label{appEq:corollary1}\\
&H(\mathbf{y}| \mathbf{v}_{\mathrm{sup}}) \leq \log 2+P_{\mathbf{v}_\mathrm{sup}}^{e} \log M. \label{appEq:corollary2}
\end{align}
   \end{linenomath}
For $H(\mathbf{y}| \mathbf{v}_{\mathrm{joint}})$ and $H(\mathbf{y}| \mathbf{v}_{\mathrm{sup}})$, we have the following relations:
\begin{linenomath}
   \begin{align}
H(\mathbf{y}| \mathbf{v}_{\mathrm{joint}}) &=H(\mathbf{y})-I(\mathbf{v}_{\mathrm{joint}} ; \mathbf{y})=H(\mathbf{y})-I(\mathbf{v}_{\mathrm{sup}} ; \mathbf{y})+I(\mathcal{G}; \mathbf{y}| \mathbf{z},k) \nonumber \\
&=H(\mathbf{y}| \mathbf{v}_{\mathrm{sup }})+I(\mathcal{G}; \mathbf{y}| \mathbf{z},k), \label{appEq:corollary}
\end{align}
   \end{linenomath}
where we use Eq.~\eqref{appEq:mutual-bound} in the second equality. Combining Eqs.~\eqref{appEq:corollary1},  ~\eqref{appEq:corollary2}  and~\eqref{appEq:corollary}, we have:
\begin{linenomath}
   \begin{align}
    \operatorname{Th}({P}_{\mathbf{v}_{joint}}^{e}) \leq \log 2+P_{\mathbf{v }_{\mathrm{sup}}}^{e} \cdot \log M+I(\mathcal{G} ; \mathbf{y}| \mathbf{z},k) \triangleq {\mathrm{RHS}}_{\mathbf{v}_{joint}}.
\end{align}
   \end{linenomath}
Further using that $I(\mathcal{G} ; \mathbf{y} \mid \mathbf{z}, k)\leq I(\mathcal{G} ; \mathbf{y}|\mathbf{z})$ and $I(\mathcal{G} ; \mathbf{y}| \mathbf{z}, k)\leq I(\mathcal{G} ; \mathbf{y}| k)$ in Eq.~\eqref{appEq:mutual-bound2}, we have 
\begin{linenomath}
   \begin{align}
 {\mathrm{RHS}}_{\mathbf{v}_{joint}}\leq \min ({\mathrm{RHS}}_{\mathbf{v}_{local}},{\mathrm{RHS}}_{\mathbf{v}_{global}}),
\end{align}
   \end{linenomath}
which completes the proof.
\end{proof}
\section{Experimental Details}
\label{app:ED}
\begin{table}[ht]
    \centering
    \caption{Statistics of used graph datasets in this paper.}
    \label{tab:stats}
    {\footnotesize
    \begin{tabular}{lccccccl}
        \toprule[1pt]
    \textbf{Dataset} & \#\textbf{Nodes} & \# \textbf{Edges} &   \#\textbf{Classes} &   \#\textbf{Features} &   \textbf{Class-average homophily}  &   {\textbf{Homophily}} \\
    \midrule
         Cora & 2,708 & 5,278 &  7 &1,433  & 0.766 & {0.81} &\\
         Citeseer & 3,327 & 4,552 &  6  & 3,703 &    0.627 & {0.74} &\\
         Pubmed & 19,717 & 44,324 &  3  & 500 & 0.664 & {0.80} &\\
         Texas & 183 & 309 &  5  & 1,793 & 0.001 & {0.11} &\\
          Cornell & 183 & 295 &  5  & 1,703 & 0.047 & {0.30} & \\
        Squirrel & 5,201 & 216,933  &  5  & 2,089 & 0.025 & {0.22} &\\
         Penn94 & 41,554 & 1,362,229  &  2  & 5 & 0.046 & {0.47} &\\
          Twitch & 9,498 & 76,569  &  2  & 2,545 & 0.142 & {0.63} & \\
   \bottomrule[1pt]
    \end{tabular}
    }
\end{table}
\subsection{Datasets Description and Statistics}
\label{app:dataset}
In our experiments, we use the following real-world datasets.

\textbf{Cora, Citeseerm, and Pubmed} are citation and high-homophily graphs, which are among the most widely used benchmarks for semi-supervised node classification~\cite{DBLP:conf/iclr/KipfW17,hamilton2017inductive,DBLP:conf/iclr/VelickovicCCRLB18}. In these graphs, nodes are documents, and edges are citations. Each node is assigned a class label based on the research field. The features of each node are represented by a bag of words of its abstracts.  

\textbf{Texas and Cornell} are low-homophily graphs representing links between web pages of the corresponding universities and originally collected by the CMU WebKB projects, where nodes and edges represent web pages and hyperlinks, respectively.  We use the pre-processed datasets by~\cite{pei2019geom}.

\textbf{Squirrel} is the subgraph of web pages in Wikipedia discussing the related topics, collected by~\cite{rozemberczki2021multi}.  The Squirrel graph is rather complex, with both homophily and heterophily combined. We also use the pre-processed datasets by~\cite{pei2019geom}

\textbf{Penn94} is a friendship network from the Facebook 100 networks of university students from
2005~\cite{lim2021large}, where nodes represent students and are labeled with the reported gender of students. The node features are major, second major/minor, dorm/house, year, and high school.

\textbf{Twitch} is a graph of relationships between accounts on the
streaming platform Twitch. Node features
are games liked, location, and streaming habits. Nodes are labeled with the explicit language used by a streamer. We utilize the pre-processed sub-graph Twitch-DE~\cite{lim2021new}. In Twitch-DE, streamers that
do not use explicit content (class 0) also often connect to streamers of class 1, which results in an overall non-homophilous structure.

We utilize the homophily~\cite{zhu2020beyond} and class-average homophily metrics proposed recently by ~\cite{lim2021large} to measure the homophily level of graphs. Specifically, the edge-homophily is the proportion of edges that connect two nodes of the same class and the class-average homophily is defined as:
\begin{linenomath}
\begin{align}
\hat{h}=\frac{1}{C-1} \sum_{k=0}^{C-1}\left[h_{k}-\frac{\left|C_{k}\right|}{N}\right]_{+},
\end{align}
\end{linenomath}
where $[a]_{+}=\max (a, 0)$ and $C$ is the total number of classes, and $h_{k}$ is the class-wise homophily:
\begin{linenomath}
\begin{align}
h_{k}=\frac{\sum_{u \in C_{k}} d_{u}^{\left(k_{u}\right)}}{\sum_{u \in C_{k}} d_{u}},
\end{align}
\end{linenomath}
in which $d_{u}$ is the number of neighbors of node $u$, and $d_{u}^{\left(k_{u}\right)}$ is the number of neighbors of $u$ that have the same class label. Compared to the edge homophily metirc~\cite{zhu2020beyond}, this metric is less sensitive to  
the number of classes and the size of each class.
The statistics of datasets are given in Table~\ref{tab:stats}. We can find that the widely used citation graphs Cora, Citeseer, and Pubmed are highly homophilous, and other graphs show low homophily degrees and thus are non-homophilous.

\subsection{Baselines and Setup}
\label{app:setup}
\textbf{Deepwalk}~\cite{perozzi2014deepwalk} is a network embedding method that generates random walks in the graph and then learns latent representations of nodes by treating walks as the equivalent of sentences with SkipGram.

\textbf{LINE}~\cite{tang2015line} is an embedding model that preserves both the first-order and second-order proximities.

\textbf{Struc2vec}~\cite{ribeiro2017struc2vec} is the  method for node representations by capturing the node structural identity.

\textbf{GAE}~\cite{kipf2016variational} is a GCN encoder trained by reconstructing adjacency matrix.

\textbf{VGAE}~\cite{kipf2016variational} is a probabilistic version of GAE and is  trained by reconstructing the adjacency matrix.

\textbf{DGI}~\cite{velivckovic2018deep} is an unsupervised node representation method that maximizes the mutual information between node representations and graph summary.

\textbf{GraphCL}~\cite{you2020graph}  learns node representations by contrasting different augmentations of graphs.

\textbf{MVGRL}~\cite{hassani2020contrastive} is a  learning method that contrasts encoders from neighbors and a graph diffusion.

\textbf{BGRL}~\cite{thakoor2021large} is a graph representation learning method that learns node representations by predicting alternative augmentations of the graph. 

We used the official implementations released by the authors for all baselines. We ran our experiments on GeForce RTX 2080 Ti. In all our experiments, we use the Adam optimizer~\cite{kingma2014adam}. The projector $g_{\theta}$ and predictor $h_{\phi}$ are both the Multilayer Perceptron (MLP) with a single hidden layer. We also used  techniques like batch normalization and negative sampling on Penn94 and Twitch datasets.
We use the same dataset splits and training procedure for all methods. We tune hyper-parameters for all models individually based on accuracy on the validation set and randomly initialize the model parameters. For all methods, the hyper-parameter search spaces are as follows: learning rate $\{0.001, 0.005, 0.01\}$, representation dimension $\{16, 32, 64\}$, L2 weight-decay $\{5e-4, 1e-4, 5e-6,1e-6\}$. For our DSSL, we tune the following hyper-parameters: $\sigma_{1}^{2} \in \{0.2, 0.4, 0.6, 0.8, 1.0\}$, $\sigma_{2}^{2} \in \{0.2, 0.4, 0.6, 0.8, 1.0\}$, $\lambda \in \{0.2, 0.4, 0.6, 0.8, 1.0\}$, $\beta \in \{0.1, 0.2, 0.4, 0.6,0.8,1.0\}$, and $K \in \{2, 4, 6, 8, 16, 32\}$. 

\subsection{Cross-Class Neighborhood Similarity}
\label{cross-neigh}
To further demonstrate how DSSL uncovers the latent patterns in neighborhood distributions, we utilize the cross-class neighborhood similarity~\cite{ma2021homophily} to serves as the ground-truth for our case study experiments in section~\ref{vis-main}. 
\begin{definition}
(Cross-Class Neighborhood Similarity). Given a graph ${G}$ and node labels, the cross-class neighborhood similarity between two classes $c, c^{\prime}$ is given by $s\left(c, c^{\prime}\right)=$ $\frac{1}{\left|\mathcal{V}_{c}\right|\left|\mathcal{V}_{c^{\prime}}\right|} \sum_{i \in \mathcal{V}_{c}, j \in \mathcal{V}_{c^{\prime}}} \cos (d(i), d(j))$ where $\mathcal{V}_{c}$ indicates the set of nodes in class $c$ and $d(i)$ denotes the empirical histogram (over $|\mathcal{C}|$ classes) of node i's neighbors' labels, and the function $\cos (\cdot, \cdot)$ measures the cosine similarity.
\end{definition}
This cross-class neighborhood similarity measures the neighborhood distributions between difference classes. If neighborhood patterns of different label nodes are different, the inter-class similarity will be low. In contrast, if nodes with the same label share the same neighborhood distributions, the intra-class similarity should be high.

\section{Additional Experimental Results}
\label{app:exp}
\subsection{Node Clustering Results}
\label{app:node-clutering}
We provide the node clustering results in Table~\ref{table:node-clustering}. As shown in 
Table~\ref{table:node-clustering}, we can find that our DSSL can consistently improve the
node clustering performance compared to the baselines on all datasets except the Cora.   This observation once again verifies the effectiveness of DSSL in learning generalizable node representations. In addition, along with the node classification performance, these results show that DSSL can provide transferable and robust node representations for various downstream tasks.
\begin{table}[t]
\centering
\caption{
Experimental results (\%) with standard deviations on the node clustering task. The best and second best
performance under each dataset are marked with boldface and underline, respectively.}\label{table:node-clustering}
\vspace{0.02in}
\scalebox{0.83}{
\begin{tabular}{ccccccccc}
\toprule
 \textbf{Method}       & \textbf{Cora}      & \textbf{Citeseer}    & \textbf{Pubmed}   & \textbf{Texas}    & \textbf{Cornell}      & \textbf{Squirrel}     & \textbf{Penn94} & \textbf{Twitch} \\
\midrule
  GAE        & \facc{37.43}{0.05}  & \facc{32.45}{0.05} & \facc{34.38}{0.06} & \facc{26.97}{0.11} & \facc{15.39}{0.13} & \facc{11.73}{0.08}  & \facc{16.70}{0.15} & \facc{27.45}{0.17}  \\ 
  VGAE        & \facc{38.92}{0.08} &  \facc{36.49}{0.03} &  \facc{41.09}{0.04} & \facc{27.75}{0.16} & \facc{17.87}{0.13} & \facc{10.83}{0.09} & \facc{17.34}{0.08} & \facc{25.89}{0.08}\\ 

  DGI        & \facc{45.17}{0.08} & \facc{42.03}{0.08} & \facc{45.33}{0.02} &  \underline{\facc{34.17}{0.07}} & \facc{15.92}{0.15} & \facc{8.49}{0.13} & \facc{12.14}{0.19} & \facc{25.84}{0.18}  \\ 
  \midrule
 GraphCL        & \facc{46.29}{0.03} &  \underline{\facc{46.38}{0.12}}  & \facc{46.78}{0.03} & \facc{30.25}{0.13} &  \facc{16.86}{0.17} & \facc{7.97}{0.13} & \facc{16.35}{0.13} & \underline{\facc{27.55}{0.16}} \\ 
   
  MVGRL        &  \bfacc{48.65}{0.07} & \facc{45.33}{0.02}  & \underline{\facc{48.81}{0.21}}  & \facc{32.72}{0.18} &  {\facc{18.02}{0.11}} &  \facc{17.65}{0.08}  & \facc{15.80}{0.17} & \facc{24.53}{0.16}  \\ 
 BGRL        &  \underline{\facc{47.35}{0.03}} & \facc{44.78}{0.04}   &  \facc{47.21}{0.07}  & \facc{33.59}{0.15} &
 
 \underline{\facc{19.74}{0.14}}
 & \facc{15.13}{0.09} & \underline{\facc{17.49}{0.13}} & \facc{25.69}{0.09} \\ 
\midrule
DSSL  &  \facc{46.77}{0.04}  &  \bfacc{47.89}{0.04} &  \bfacc{49.39}{0.02} &  \bfacc{38.22}{0.15} &  \bfacc{20.36}{0.08}  &   \bfacc{19.85}{0.13} &  \bfacc{18.66}{0.15}  &   \bfacc{29.63}{0.08} \\ 
\bottomrule
\end{tabular}}
\vspace{-0.1in}
\end{table}

\subsection{{Additional Results on Synthetic Dataset}}
In this section, we generate the synthetic graphs with various homophily ratios $h$ to evaluate DSSL. Specifically, we consider the synthetic dataset: syn-products generated by~\cite{zhu2020beyond}. We refer the interested reader to~\cite{zhu2020beyond} for the detailed generation process. The statistics for synthetic dataset syn-products are provided in 
Table~\ref{tab:synthetic}. Roughly speaking, the synthetic graphs can be generated by controlling various edge-homophily ratios $h$  where $h=1$ indicts strong edge-homophily, and $h=0$ indicts perfect heterophily of generated graph. Here, we consider six edge-homophily ratios, $h=[0.0,0.25,0.50,0.75,1.0]$. We randomly split 25\% of nodes into the training set, 25\% into validation, and 50\% into the test set for the downstream linear classifier. Table~\ref{table:syn} gives the results on syn-products dataset. From this table, we  have the following observations:
(1) GAE achieves the best accuracy under strong homophily ($h=1$), but it is significantly less accurate than other methods in low-homophily settings. This is reasonable since GAE over-emphasizes proximity and the homophily assumption. (2) Our DSSL outperforms recent contrastive learning methods such as DGI and MVGRL, which shows the effectiveness of DSSL on non-homophily scenarios, and suggests that designing effective graph augmentation for those contrastive baselines on the non-homophilous graphs is difficult due to the heterogeneous and diverse neighbor pattern. (3) Our DSSL can generally achieve the best performance and effectively adapt to all levels of homophily.


\begin{table}[ht]
    \centering
    \caption{{Statistics of the synthetic dataset: syn-products.}}
    \label{tab:synthetic}
    {\footnotesize
    \begin{tabular}{lccccccc}
        \toprule[1pt]
    \textbf{Dataset} & \# \textbf{Nodes} & \# \textbf{Edges} &   \# \textbf{Classes} &   \# \textbf{Features}  &   {\textbf{Edge homophily}} \\
    \midrule
         syn-products & 10000 & 59,640 to 59,648 &  10 & 100  & [0.0, 0.25, 0.50, 0.75, 1.0] & \\
   \bottomrule[1pt]
    \end{tabular}
    }
\end{table}

\begin{table}[t]
\centering
\caption{{Experimental results (\%) on   synthetic datasets, syn-products, measured in terms of node classification accuracy along with standard deviations.  The best and second best performances under each dataset are marked with  boldfac and underline, respectively.}}\label{table:syn}
\vspace{0.02in}
\scalebox{0.9}{
\begin{tabular}{lccccc}
\toprule
 \textbf{Homophily ratio $h$}       & \textbf{0.00}   & \textbf{0.25}    &  \textbf{0.50}  & \textbf{0.75}    &  \textbf{1.00}     \\
\midrule
GAE  &  \facc{18.28}{0.52}  &  \facc{25.49}{0.37} &  \facc{50.62}{0.28} &  \facc{82.05}{0.38} &   \bfacc{95.27}{0.11}   \\ 
DGI  &  \facc{40.21}{0.43}  &  \facc{42.38}{0.26} &  {\facc{57.31}{0.36}} &  {\facc{85.21}{0.43}} &  {\facc{92.12}{0.43}}    \\ 
MVGRL  &  \facc{42.69}{0.32}  &  \facc{44.35}{0.40} &  {\facc{59.43}{0.27}} &  \underline{\facc{86.65}{0.39}} &  {\facc{93.11}{0.35}}    \\
BGRL  &  \underline{\facc{45.31}{0.47}}  &  \underline{\facc{47.52}{0.38}} &  \underline{\facc{61.37}{0.26}} &  {\facc{86.21}{0.24}} &  {\facc{93.28}{0.37}}      \\ 
DSSL  &  \bfacc{49.19}{0.62}  &   \bfacc{51.35}{0.51} &  {\bfacc{64.78}{0.25}} &   {\bfacc{87.59}{0.46}} &  \underline{\facc{94.83}{0.55}} \\     
\bottomrule
\end{tabular}}
\vspace{-0.1in}
\end{table}

	
		

\begin{figure}
	\centering
	\includegraphics[width=1.0\linewidth]{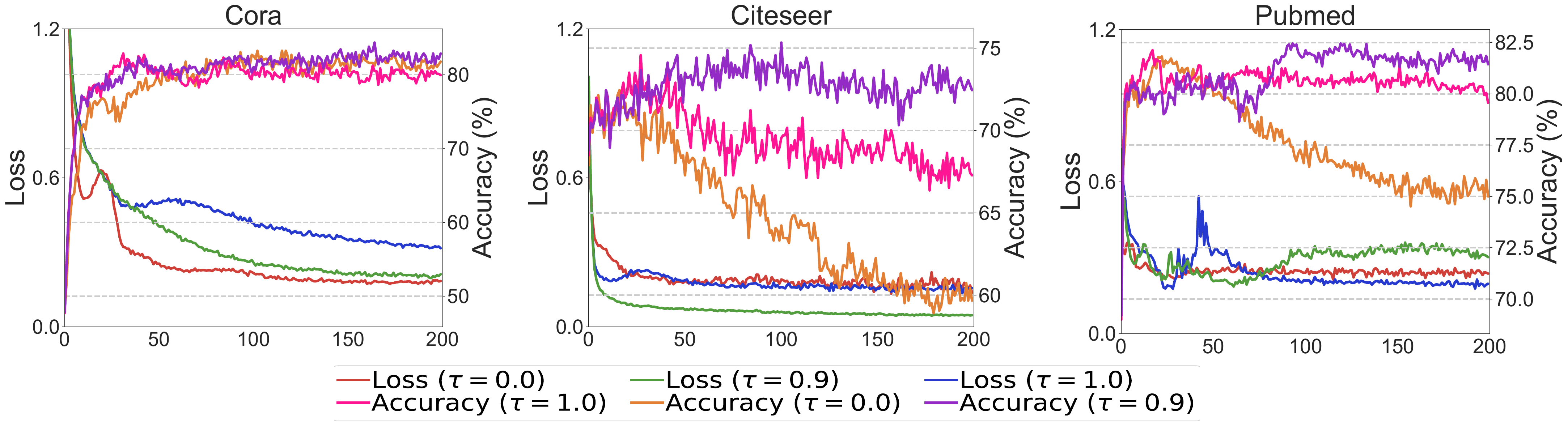}
	\vskip -0.5em
	\caption{Performance on Cora, Citeseer and Pubmed with varying $\tau$.}
	\label{fig:appendix_ab1}
\end{figure}

\begin{figure}
	\centering
	\includegraphics[width=1.0\linewidth]{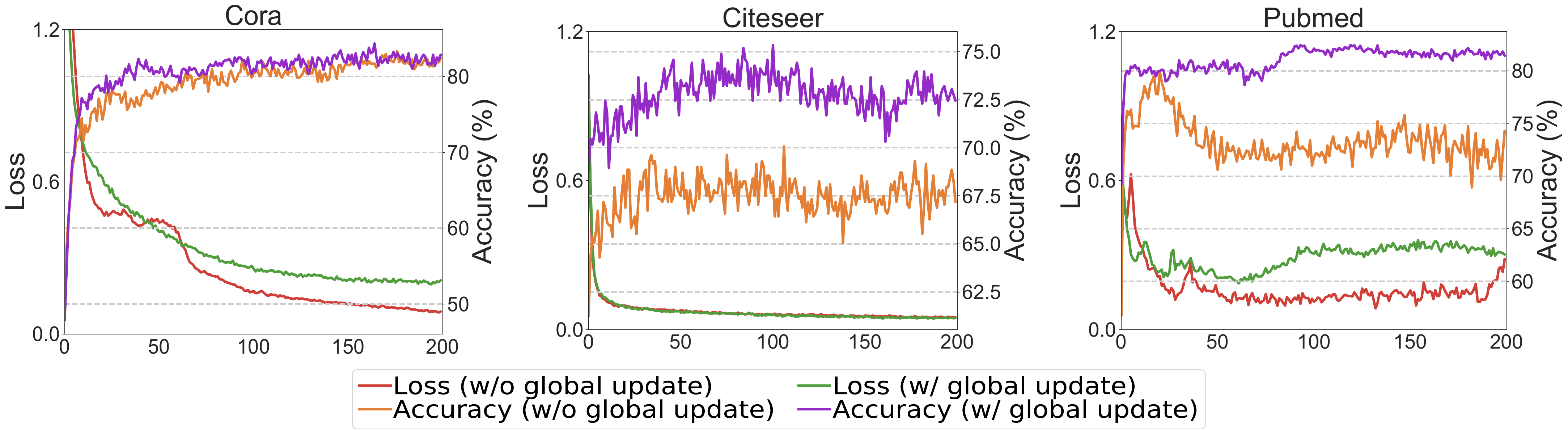}
	\vskip -0.5em
	\caption{Performance on Cora, Citeseer and Pubmed w/ and w/o global update.}
	\label{fig:appendix_ab2}
	\vskip -1em
\end{figure}

\begin{figure*}
	\centering
	\includegraphics[width=1.0\linewidth]{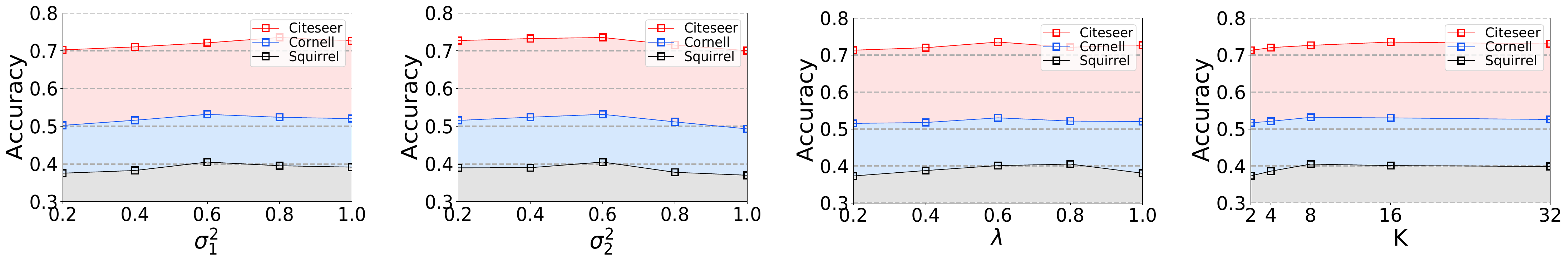}
	\vskip -1em
	\caption{ Hyper-parameter analysis  on Citeseer, Cornell and Squirrel datasets.}\label{fig:parameters-app}
\end{figure*}

\subsection{Additional Results on Ablation Study and Parameter Analysis}
\label{app:ab}
\textbf{Ablation Study}. 
We further conduct experiments to study the effects of the global update on prototypes and asymmetric architecture on homophily datasets.
Figures~\ref{fig:appendix_ab1} and ~\ref{fig:appendix_ab2} show the ablation results on Cora, Citeseer and Pubmed.  
These results also echo the collapse issue as discussed in the main text: without the global update and asymmetric architecture, there is 
a significant performance drop on these three homophily datasets.

\textbf{Parameter Analysis}. 
We provide more results on the parameter analysis. Figure~\ref{fig:parameters-app} shows the results on  Citeseer, Cornell, and Squirrel datasets. We can observe that having large values of $\sigma_2^{2}$ also does not improve the overall performance, and  DSSL is not very sensitive to the Gumbel softmax temperature $\lambda$ for these three datasets. We also find that as $K$ increases from $2$ to $8$ ($2$ to $16$ for Citeseer), the accuracy
of DSSL improves.

\begin{figure}[t]
	\centering
	\includegraphics[width=0.95\linewidth]{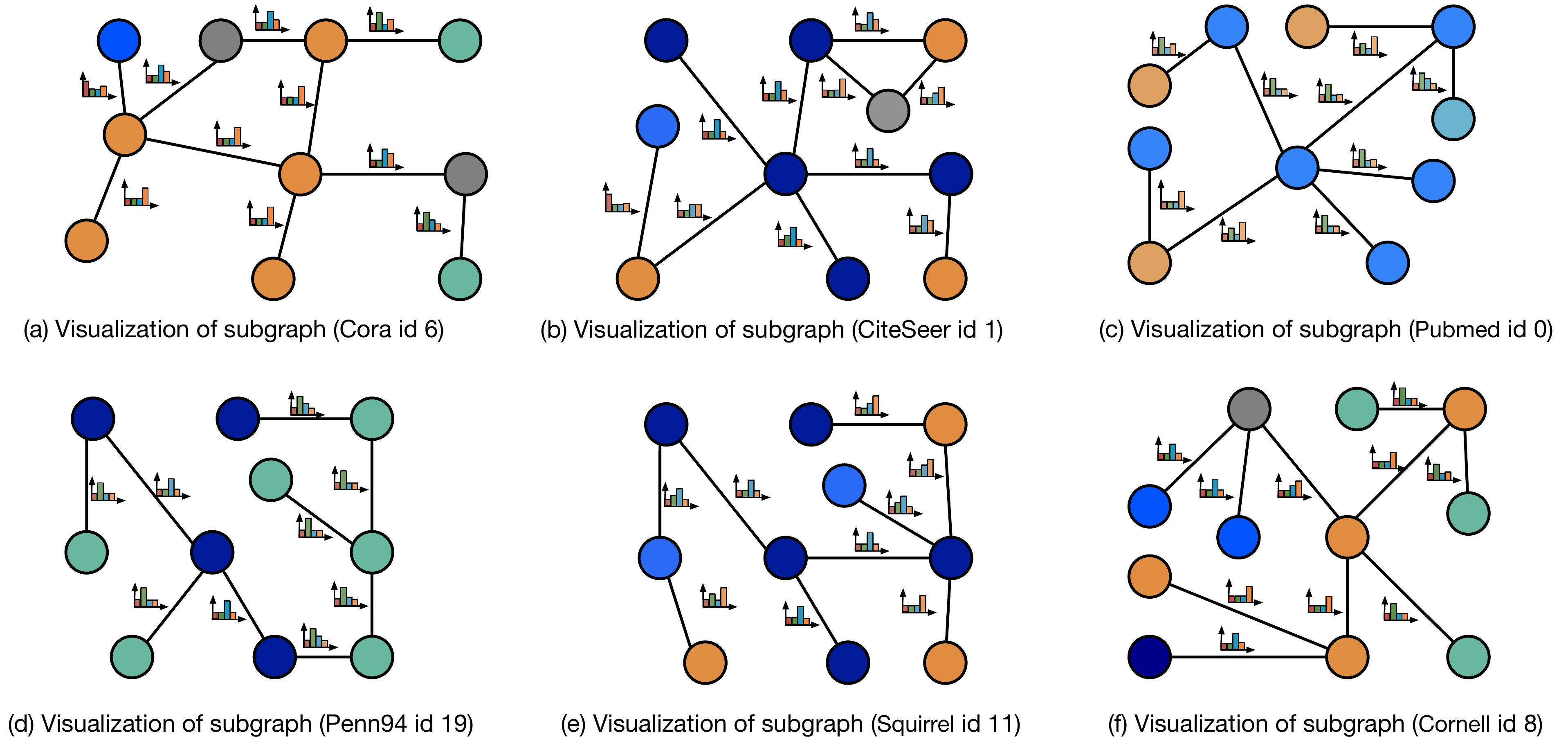}
	\vskip -0.5em
	\caption{The visualization and case study results on other datasets. (best
viewed on a computer screen and note that the latent distribution of each link need
to be zoomed in to be better visible).}
	\label{fig:appendix_visu}
	\vskip -1em
\end{figure}

\begin{figure}[t]
	\centering
	\includegraphics[width=0.95\linewidth]{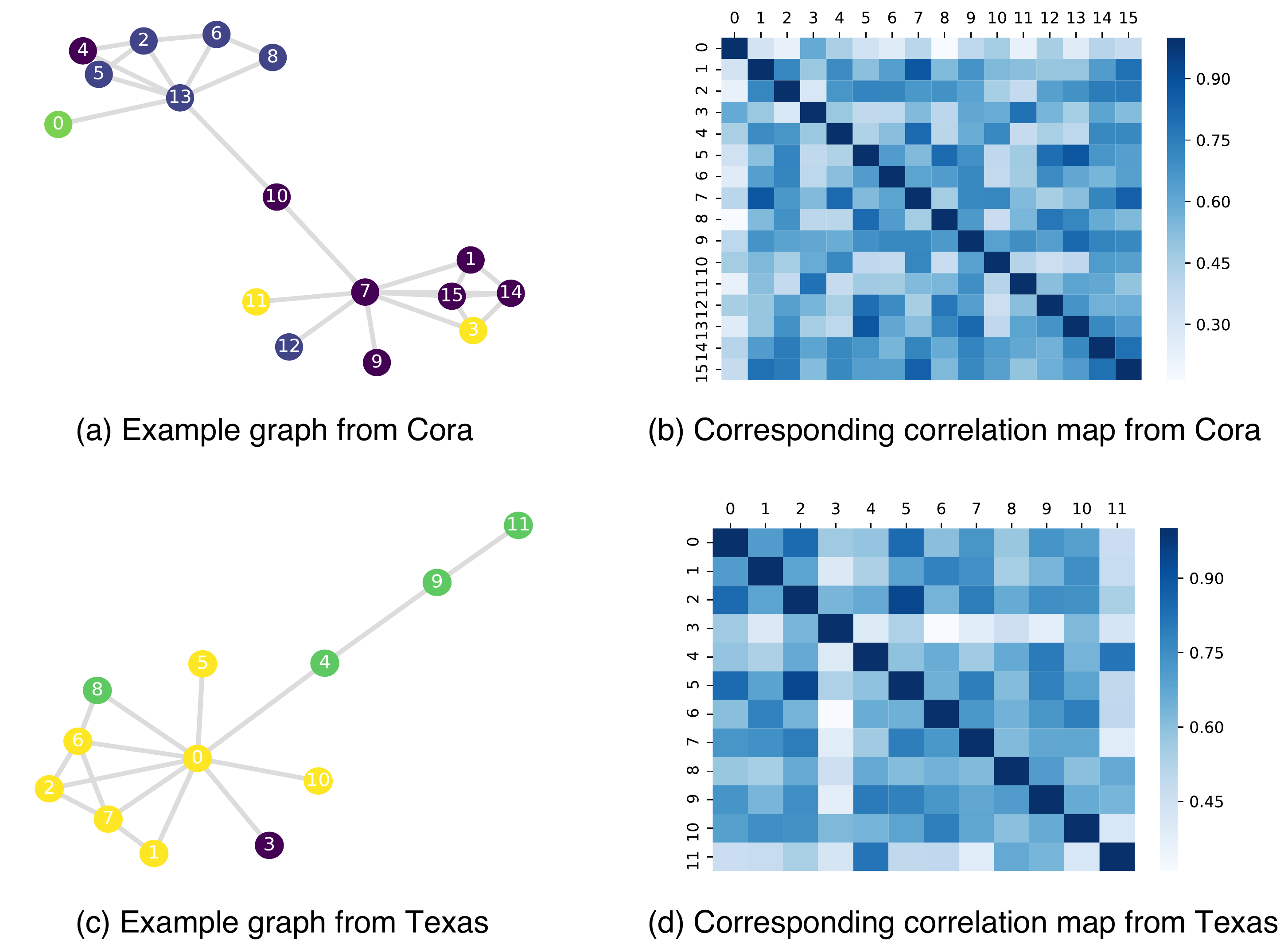}
	\vskip -0.5em
	\caption{Example graphs and correlation maps. The graphs are randomly sampled
from Cora and Texas. The correlation maps are obtained based on the pair-wise similarity of learned posteriors.}
	\label{fig:long-cora-texas}
	\vskip -1em
\end{figure}


\subsection{Additional Results on Visualization and Case Study}
\label{app:vis-cs}
To further examine the learned latent factors, in Figure~\ref{fig:appendix_visu}, we provide the additional case study and visualization results on other datasets. We can find that, in most cases, the latent factors share a similar distribution for the same type of link, which illustrates its effectiveness in decoupling the
underlying latent semantic meaning of different neighbors. This also interprets the reason why capturing the semantic shift can improve performance. Figure~\ref{fig:long-cora-texas} shows the results of long-range nodes on Cora and Texas datasets. We can also find that some nodes exhibit similar semantic clusters, regardless of the distance between them. These additional qualitative visualization results, along with our other experiment results, suggests that considering the semantic shift between local neighbors and capturing global semantics may be necessary for non-homophilous graphs.

\section{Discussion and Broader impact}
\label{app:impact}
In this paper, we proposed a new framework called decoupled self-supervised learning (DSSL) for unsupervised node representation learning on non-homophilous graphs. The theoretical analysis and extensive experiments show the effectiveness of our proposed DSSL. 

\textbf{Limitation of the work}. 
Despite the theoretical grounds and the promising experimental justifications, there is one limitation of the current work, which we hope to improve in future work: our theoretic analysis requires some assumptions on the relationship between the downstream labels and the neighborhood distribution. While, not surprisingly, we have to make some assumptions to expect good generalization for the unsupervised node representation learning, it is an interesting future direction to explore more relaxed assumptions than the ones used in this work.

\textbf{Potential negative societal impacts}. 
In this work, we propose a self-supervised learning framework for node representation learning which does not rely on annotated labels, which might reduce the need for label annotation and thus makes a few individuals who focus on labeling or annotating data unemployed. In addition, our learned representation may suffer from malicious adversaries who seek to extract node-level information about sensitive attributes. Thus, improving the robustness of DSSL is an interesting future work direction.

\end{document}